\documentclass{article}

\usepackage[preprint]{neurips_2026}

\usepackage[utf8]{inputenc} % allow utf-8 input
\usepackage[T1]{fontenc}    % use 8-bit T1 fonts
\usepackage{hyperref}       % hyperlinks
\usepackage{url}            % simple URL typesetting
\usepackage{booktabs}       % professional-quality tables
\usepackage{amsfonts}       % blackboard math symbols
\usepackage{nicefrac}       % compact symbols for 1/2, etc.
\usepackage{microtype}      % microtypography
\usepackage{xcolor}         % colors

\title{Fast and Stable Triangular Inversion for Delta-Rule Linear Transformers}

\author{%
  Aleksandros Sobczyk 
  \qquad
  Gioele Gottardo  
  \qquad
  Christos K. Matzoros 
  \AND
  Mirko De Vita
  \qquad
  Filip Skogh 
  \qquad
  Anastasios Zouzias 
  \qquad
  Jiawei Zhuang \\ \\
  Huawei Technologies Switzerland AG \\
  Zurich, Switzerland \thanks{Contacts: \quad \texttt{aleksandros.sobczyk@huawei.com} \quad \texttt{gioele.gottardo2@h-partners.com}\\ 
  \texttt{christos.konstantinos.matzoros@h-partners.com} 
  \quad \texttt{mirko.de.vita@h-partners.com} \\
  \texttt{filip.skogh@h-partners.com}\quad
  \texttt{anastasios.zouzias@huawei.com}\quad 
  \texttt{zhuangjiawei@huawei.com}
  }
}

\usepackage{amsmath}
\usepackage{amssymb}
\usepackage{amsthm}
\usepackage{subcaption}
\usepackage{multirow}

\newcommand{\sm}{\left(\begin{smallmatrix}}
\newcommand{\esm}{\end{smallmatrix}\right)}

\hypersetup{
     colorlinks=true,
     linkcolor=blue,
     citecolor=blue,
     urlcolor=blue,
}

\usepackage{bm} % fixes boldsymbol
\usepackage{amsmath} % Comment out for acmart
\usepackage{amssymb} % Comment out for acmart
\usepackage{amsthm} % Comment out for acmart
\usepackage{algorithm}
\usepackage[noend]{algpseudocode}
\usepackage{braket}

\algnewcommand\algorithmicinput{\textbf{Input:}}
\algnewcommand\Input{\item[\algorithmicinput]}
\algnewcommand\algorithmicoutput{\textbf{Output:}}
\algnewcommand\Output{\item[\algorithmicoutput]}
\algnewcommand\algorithmicalgorithm{\textbf{Algorithm:}}
\algnewcommand\Algorithm{\item[\algorithmicalgorithm]}

\newtheorem{theorem}{Theorem}[section]
\newtheorem{lemma}{Lemma}[section]
\newtheorem{proposition}{Proposition}[section]
\newtheorem{corollary}{Corollary}[section]

\newtheorem{definition}{Definition}[section]

\DeclareMathOperator{\diag}{diag}

\DeclareMathOperator{\poly}{poly}
\DeclareMathOperator{\polylog}{polylog}

\newcommand\veca{\boldsymbol{\mathrm{a}}}
\newcommand\vecb{\boldsymbol{\mathrm{b}}}

\newcommand\vece{\boldsymbol{\mathrm{e}}}

\newcommand\veck{\boldsymbol{\mathrm{k}}}

\newcommand\veco{\boldsymbol{\mathrm{o}}}

\newcommand\vecq{\boldsymbol{\mathrm{q}}}

\newcommand\vecu{\boldsymbol{\mathrm{u}}}
\newcommand\vecv{\boldsymbol{\mathrm{v}}}

\newcommand\vecx{\boldsymbol{\mathrm{x}}}

\newcommand\vecgamma{\boldsymbol{\gamma}}

\newcommand\matA{\boldsymbol{\mathrm{A}}}
\newcommand\matB{\boldsymbol{\mathrm{B}}}

\newcommand\matD{\boldsymbol{\mathrm{D}}}
\newcommand\matE{\boldsymbol{\mathrm{E}}}
\newcommand\matF{\boldsymbol{\mathrm{F}}}
\newcommand\matG{\boldsymbol{\mathrm{G}}}
\newcommand\matH{\boldsymbol{\mathrm{H}}}
\newcommand\matI{\boldsymbol{\mathrm{I}}}

\newcommand\matK{\boldsymbol{\mathrm{K}}}
\newcommand\matL{\boldsymbol{\mathrm{L}}}
\newcommand\matM{\boldsymbol{\mathrm{M}}}

\newcommand\matO{\boldsymbol{\mathrm{O}}}

\newcommand\matQ{\boldsymbol{\mathrm{Q}}}
\newcommand\matR{\boldsymbol{\mathrm{R}}}
\newcommand\matS{\boldsymbol{\mathrm{S}}}
\newcommand\matT{\boldsymbol{\mathrm{T}}}
\newcommand\matU{\boldsymbol{\mathrm{U}}}
\newcommand\matV{\boldsymbol{\mathrm{V}}}
\newcommand\matW{\boldsymbol{\mathrm{W}}}
\newcommand\matX{\boldsymbol{\mathrm{X}}}
\newcommand\matY{\boldsymbol{\mathrm{Y}}}
\newcommand\matZ{\boldsymbol{\mathrm{Z}}}

\newcommand\matGamma{\boldsymbol{\mathrm{\Gamma}}}
\newcommand\matDelta{\boldsymbol{\mathrm{\Delta}}}

\newcommand\matXi{\boldsymbol{\mathrm{\Xi}}}

\newcommand\matAtilde{\widetilde{\boldsymbol{\mathrm{A}}}}

\newcommand\matCtilde{\widetilde{\boldsymbol{\mathrm{C}}}}

\newcommand\matXtilde{\widetilde{\boldsymbol{\mathrm{X}}}}
\newcommand\matYtilde{\widetilde{\boldsymbol{\mathrm{Y}}}}
\newcommand\matZtilde{\widetilde{\boldsymbol{\mathrm{Z}}}}

\newcommand\matAhat{\widehat{\boldsymbol{\mathrm{A}}}}

\newcommand\matXhat{\widehat{\boldsymbol{\mathrm{X}}}}

\newcommand{\fl}{\boldsymbol{\mathsf{fl}}}
\newcommand{\umach}{\textbf{\textup{u}}}
\newcommand\RR{\mathbb{R}}

\newcommand{\norm}[1]{\ensuremath{\left\|#1\right\|_2}}

\begin{document}

\maketitle

\begin{abstract}
  Linear attention has emerged as a cornerstone for efficient long-context architectures, as evidenced by its integration into state-of-the-art open-source models including Qwen3.5/3.6, Kimi Linear/K3, and RWKV-7. Models that incorporate linear attention layers with the so-called Delta-Rule involve the inversion of triangular matrices as a core sub-routine. This operation often forms a performance bottleneck, and, due to its high-sensitivity to numerical errors, it can significantly deteriorate end-to-end model accuracy if it is not carefully implemented. This work provides a systematic analysis of both direct and iterative triangular inversion algorithms, targeting methods that are rich in matrix products, and, therefore, have the potential to efficiently utilize modern hardware. To that end, our analysis covers a broad spectrum of mathematical and practical aspects, with a heavy focus on numerical stability, computational complexity, and, ultimately, hardware efficiency and practical considerations. We provide a rigorous experimental evaluation to verify these properties in practical scenarios, and in low-precision floating-point representations, highlighting the strengths and limitations of each method. Performance benchmarks on NPUs reveal up to $4.3\times$ speed-up against the state-of-the-art implementations of SGLang for triangular matrix inversion, leading to significant performance improvements on the entire layer level, while maintaining end-to-end model accuracy. 
\end{abstract}
\ \\\ \\\ \\
\textit{
    \small
    \textbf{Publication notice:} This article has been accepted for publication in NeurIPS 2026. This preprint version contains further extensions and results that were developed while the manuscript was under peer-review. The following additional contents are included:
    \begin{itemize}
        \item Stability analysis of Algorithm \ref{alg:mch} (MCH) in Lemma \ref{lemma:stability_of_mch}.
        \item Stability analysis of Newton-Schulz in Lemma \ref{lemma:stability_of_ns} and Theorem \ref{theorem:stability_of_ns}.
        \item Analysis of the condition number of GDN / KDA-type matrices in Proposition \ref{proposition:condition_number_upper_bound_delta_matrices}.
        \end{itemize}
}

\newpage
\tableofcontents

\newpage
\section{Introduction}
The attention mechanism \cite{vaswani2017attention} is one of the most important innovations of the last decade in deep learning, and lies at the core of Large Language Models (LLMs).
The quadratic time complexity and the linearly growing KV-cache of softmax-attention have become a key bottlenecks for scaling LLMs. Linear attention \cite{katharopoulos2020transformers} offers a way to reduce the computational complexity, and many different variants have been developed to close the performance gap between linear and standard attention \cite{dao2024transformers,kimi2025kimi,peng2025rwkv,yang2023gated}. Prominent approaches towards this direction include DeltaNets \cite{schlag2021linear,yang2024parallelizing} and Gated DeltaNet (GDNs) \cite{yang2025gated}, which introduce decay and gating factors. 

In brief, given an input sequence of length $L$, the corresponding query-key-value matrices $\matQ,\matK,\matV\in\mathbb{R}^{L\times d}$, where $d$ is the head dimension, and a causal mask (lower-triangular) matrix $\matM\in\mathbb{R}^{L\times L}$, in its simplest form linear attention can be defined in two equivalent ways:
%
%%
%%%
\begin{align}
    \matO&=(\matQ\matK^\top \odot\matM)\matV\quad \text{(matrix form)}, \label{eq:linear_attention_matrix_form}\\
    \veco_t&=\left(\sum_{j=1}^{t}\vecv_j\veck_j^\top \right)\vecq_t= \matV^\top_{1:t,:}\matK_{1:t,:}\vecq_t\quad \text{(linear recurrence)},
    \label{eq:linear_attention_recurrent_form}
\end{align}
%%%
%%
%
where $\odot$ denotes element-wise matrix product. The linear recurrence formulation requires significantly less memory than the matrix form, as it only requires the state matrix to update the model state at any time step. On the other hand, the matrix form lends itself for parallel processing. To achieve the ``best of both worlds'', and to optimize the computational capabilities of AI accelerators, a standard technique is to utilize \emph{chunk-wise} parallel processing~\cite{hua2022transformer, yang2024parallelizing}. In this case, the sequence is partitioned into $\lceil L/C\rceil$ ``chunks'' of (at most) length $C$ each. In practice, $C$ is typically one of $\{16,32,64,128\}$. 
Let us denote the corresponding ``chunks'' of the matrix $\matK$ as $\matK_{[k]}:=\matK_{kC:(k+1)C,:}\in\mathbb{R}^{C\times d}$.
After several non-trivial mathematical re-formulations (see e.g. \cite{yang2025gated}), the final output updates within each chunk require the computation of the inverse of a \emph{unit-diagonal lower triangular matrix} $\matA$ of size $C\times C$, which has the form:
%
%%
%%%
\begin{align}
    \matA_{[k]}^{-1} = \left(
        \matI + \phi\left(\matK_{[k]}, \matK^\top_{[k]}\right) \odot \matL^{-}
        \right)^{-1},\quad \matL^{-}_{i,j}:=\begin{cases}
    0, \text{ if } i\leq j,\\
    1, \text{ otherwise},
\end{cases}
\label{equation:definition_of_lower_triangular_matrices}
\end{align}
%%%
%%
%
where $\phi\left(\cdot, \cdot\right)\in{\RR^{C\times C}}$ is a  function that depends on the specific model architecture,  typically involving element-wise functions of the elements of $\matK$, such as scaling or exponentiation, $\matI$ is the identity matrix, and $\matL^{-}$ is a strictly lower triangular all-ones (mask) matrix of size $C$. We provide a description of the different $\phi(\cdot,\cdot)$ functions that are used by various models in Table \ref{table:inverse_functions_of_different_models}.
%

%
%%
%%%
\begin{minipage}{0.32\textwidth}
\setlength{\tabcolsep}{1.2pt}
\begin{table}[H]
    \centering
    \caption{Examples of functions $\phi(\cdot, \cdot)$ in existing models.}
    \footnotesize
    \renewcommand{\arraystretch}{1.1}
    \begin{tabular}{r l}
    \hline
    Model & $\phi\left(\cdot , \cdot \right)$ \\\hline\hline\noalign{\vskip 2pt}
        DeltaNet 
        \cite{schlag2021linear} & 
        $\diag(\beta)\cdot \matK\matK^\top$  
    \\
        DeltaFormer \cite{xu2025deltaformer}  & 
        $\exp(\matW\matK^{\top})$  
    \\
        GDN \cite{yang2025gated}&
        $\matK\matK^\top \odot \mathcal{A}$ 
    \\
        Comba \cite{hu2026improving}& 
        $\matK\matK^\top \odot \mathcal{A}^{i-1/j}$
    \\
        KDA \cite{kimi2025kimi}&
        $(\matK\odot \matGamma)(\matK \oslash \matGamma)^\top$
    \\\hline
    \end{tabular}
    \label{table:inverse_functions_of_different_models}
\end{table}
\end{minipage}
\hfill
\begin{minipage}{0.65\textwidth}
\setlength{\tabcolsep}{3pt}
\begin{table}[H]
    \centering
    \caption{Summary of triangular inversion methods.}
    \label{table:summary_of_methods}
    \footnotesize
    \begin{tabular}{l c l r l}
    \hline
    Method (Acronym) & Type & Stability & \# MatMuls & Refs.\\
    \hline\hline
    Vector Col. Sweep (VCS)  & Direct & stable (Lem.~\ref{lemma:stability_of_column_sweep}) & -- & \cite{gallopoulos2016parallelism,sameh1977solving} \\
    Matrix Col. Sweep (MCS) & Direct & stable (Lem.~\ref{lemma:stability_of_column_sweep}) & $O(n)$ & \cite{gallopoulos2016parallelism,sameh1977solving} \\
    Bunch-Hopcroft (MBH) & Direct & log-stable  & $\approx 2\log(n)$ & \cite{bunch1974triangular,demmel2007fast} \\
    Cayley-Hamilton (MCH) & Direct & unstable & $2\lceil \log(\tfrac{n}{2})\rceil$ & \cite{codenotti2001role} \\
    Mixed Recursion (MXR) & Direct & log-stable & $\approx2\log(n)$ & Sec. \ref{section:mixed_recursive_algorithm} \\
    Newton-Schulz (NS) & Iterative & stable & $2\times N_{\text{iter}}$ & Sec. \ref{section:newton_schulz} \\
    Iterative Refinement (IR) & Iterative & stable (per iter) & $2\times N_{\text{iter}}$ & \cite{higham2002accuracy} \\\hline
    \end{tabular}
\end{table}
\end{minipage}
%%%
%%
%
%
%%
%%%
\subsection{Motivation and Contributions}
%%%
%%
%
In existing model implementations, the actual inversion of the individual chunks can quickly become a performance bottleneck. For example, in Figure \ref{fig:gdn_runtime_breakdown} (Left) it is shown that the triangular inversion part can take up to $\sim 40\%$ of the total layer device-runtime.
\begin{figure}[htb]
    \centering
    \includegraphics[width=0.43\linewidth]{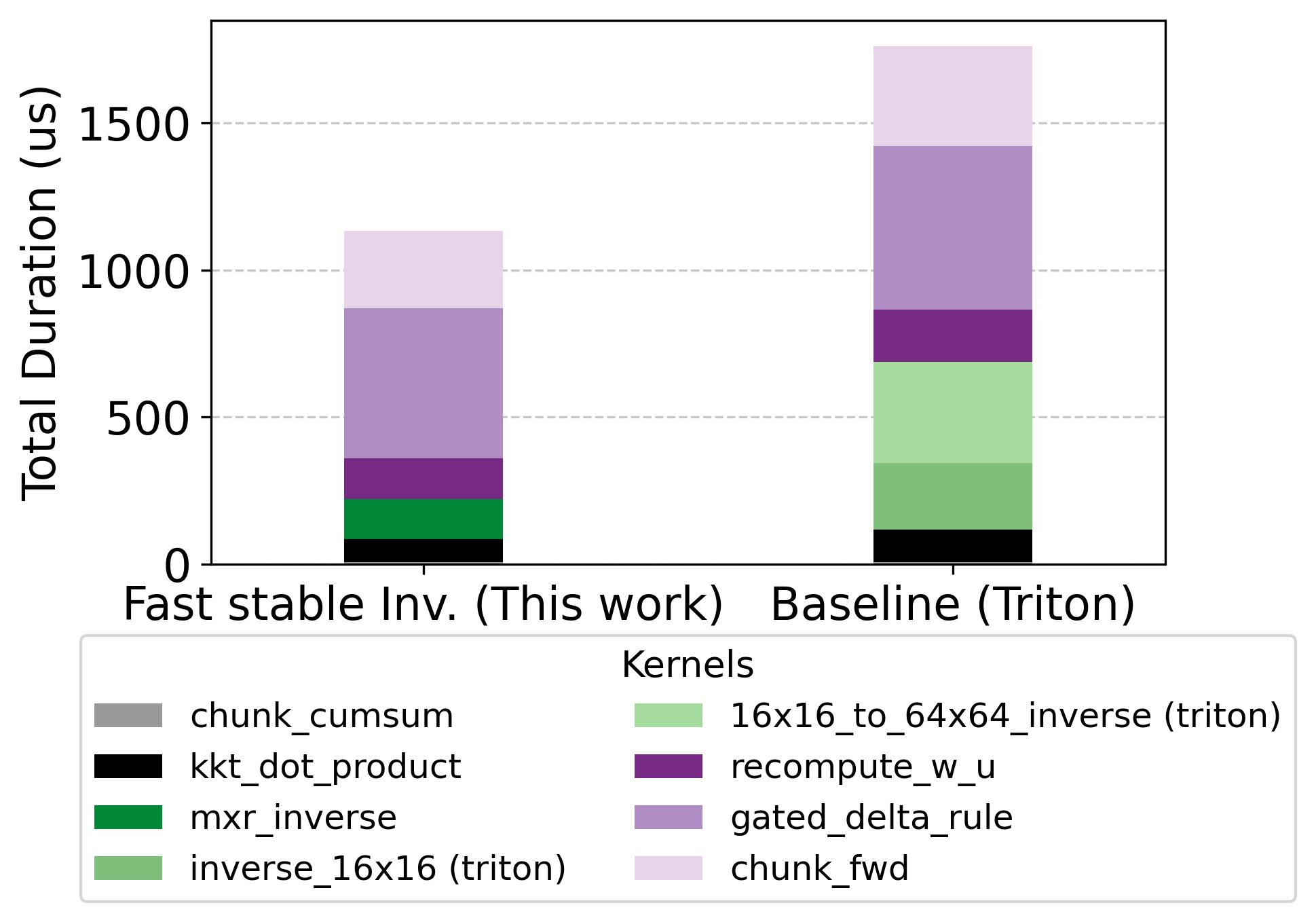}
    \quad
    \includegraphics[width=0.53\linewidth]{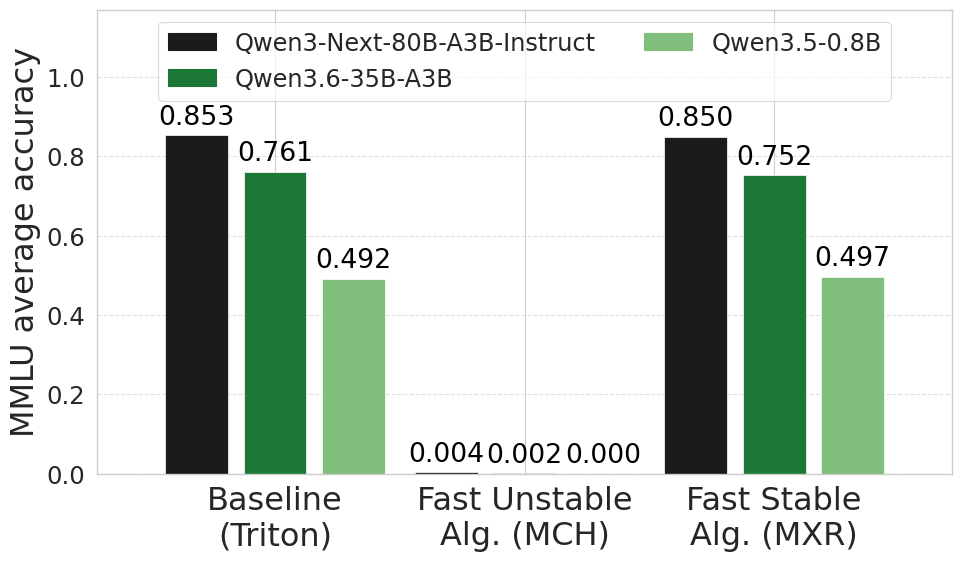}
    \caption{Left: Runtime breakdown of single GDN layer. Right: Accuracy drop in end-to-end benchmark evaluation due to inversion error accumulation.}
    \label{fig:gdn_runtime_breakdown}
\end{figure}

Alternative approaches for hardware-aware algorithms have been proposed \cite{zhong2025understanding} to better utilize matrix cores in AI accelerators and to surpass the performance bottlenecks. However, matrix cores are usually optimized for low-precision (half) arithmetic, and hence if these hardware-aware algorithms are not designed carefully, the errors introduced due to low-precision floating point arithmetic can easily lead to catastrophic accumulation of errors, often returning completely wrong, or even "NaN" solutions. This is clearly illustrated in Figure \ref{fig:gdn_runtime_breakdown} (Right), where we show how an unstable inversion method drastically reduces the reasoning capabilities of Qwen3.6, dropping the accuracy on MMLU from 76.1\% to 0.2\%! Evidently, model accuracy is highly correlated to the accuracy of the underlying inversion algorithm.

This is of course not surprising. Inverting triangular matrices, or, even more importantly,\footnote{In many applications, the main task is to solve a linear system of equations. In such cases the explicit computation of the entire inverse is evidently redundant, and even characterized as a ``sin'' \cite{higham2022seven}, as it can harm both the performance and the numerical accuracy.} solving triangular systems of linear equations, is a very thoroughly studied problem in numerical analysis, dating all the way back to the early works of Wilkinson \cite{wilkinson1961error}. One of the main reasons is that triangular factorizations (LU, Cholesky, QR, etc) established themselves early-on as the ``gold standard'' procedure to solve general (dense) linear systems of equations. As such, the stability properties and computational complexity are quite well-understood, and algorithm implementations should always take them into consideration and carefully handle numerical errors.

The aforementioned observations motivated us to revisit the rich bibliography on triangular matrix inversion algorithms, in the context of Linear Transformers. 
We focus on deriving algorithms based on the following guidelines:
\begin{description}
    \item[Matrix products and hardware efficiency:] In order to take advantage of matrix cores in modern AI accelerators, the underlying inversion algorithms should extensively use matrix products.
    \item[Numerical stability:] Fast algorithms are not useful if they return wrong results. To that end, numerical stability analysis is mandatory to ensure high-accuracy and meaningful results in real-applications. 
    \item[Focus on Linear Attention:] The specific properties of the matrices arising in Delta-Rule Linear Transformers should be exploited as much as possible to achieve the best performance, both from the mathematical analysis perspective, as well as practical performance.
\end{description}

By dedicating the analysis to the underlying types of matrices, we provide some novel mathematical properties that are also useful for practical applications. Stability properties are thoroughly tested experimentally for all algorithms, in three different precisions: float32, float16, and bfloat16. Interestingly, we show that even \emph{abysmally unstable} algorithms (Alg. \ref{alg:mch}, MCH), if used carefully, can be efficiently exploited as sub-routines inside very fast and stable algorithms (Alg. \ref{alg:mxr}, MXR). In terms of performance, the latter approach achieved up to 4.3x speedup on NPUs compared to the current state-of-the-art implementations for triangular inversion in SGLang.

\subsection{Acknowledgements}
We thank Efstratios Gallopoulos for helpful literature recommendations, and for suggesting the term ``abominable'' to characterize the condition number of random triangular matrices. We also thank the anonymous referees and area chairs of NeurIPS 2026 for their helpful feedback and constructive criticism.

\subsection{Code availability}
\label{section:source_code}
The source code for the kernels and the scripts to reproduce  the experimental evaluations are available in the following links:
\begin{itemize}
    \item \hyperlink{https://github.com/huawei-csl/pto-kernels}{https://github.com/huawei-csl/pto-kernels}
    \item \hyperlink{https://github.com/huawei-csl/gdn-tri-inverse}{https://github.com/huawei-csl/gdn-tri-inverse}
    \item \hyperlink{https://github.com/huawei-csl/megagdn-pto}{https://github.com/huawei-csl/megagdn-pto}
\end{itemize}

\subsection*{Declaration of LLM usage}
LLMs were used for assistance in writing and coding. They did not influence the core methodology, scientific rigor, or originality of the research.

\section{Preliminaries}

\subsection{Floating point model}
\label{section:floating_point_model}
Floating point error analysis becomes more and more important in AI computing, especially with the recent trend of reducing the bits of precision in floating point formats. For the purposes of this note, we recall some basic definitions, and refer to the classic textbook of Higham~\cite{higham2002accuracy} for further details.

In the standard floating point model, the floating point representation $\fl(x)$ of a real number $x$ is a $(1+p+t)$-bit number:
\begin{equation}
\fl(x) = \pm 2^e \times \left(\tfrac{m_1}{2} + \tfrac{m_2}{2^2} + \cdots + \tfrac{m_t}{2^t}\right),
\end{equation}
where one bit stores the sign of the number $\pm$, $p$ bits are used to store the (integral) exponent $e$ in the so-called ``biased'' format, and $t$ bits $m_1, m_2, \dots, m_t$ are used for the \emph{significand} (also known as the \emph{mantissa}). For all \emph{normalized} numbers it holds that:
\begin{equation}
\fl(x) = (1+\theta)x,\quad 
|\theta| \leq 2^{-t}:=\umach,
\end{equation}
where $\umach$ is referred to as the \emph{machine precision}. Normalized numbers are all numbers that can be represented within the given exponent range. Given two (normalized) floating point numbers $a,b$, floating point operations $\circ \in \{+,-,\times,/\}$ and square roots are assumed to satisfy:
\begin{equation}
\fl(a \circ b) = (1+\theta)(a \circ b),
\qquad
\text{and}
\qquad
\fl\left(\sqrt{a}\right) = (1+\theta)\sqrt{a}.
\end{equation}

These relations can be generalized to matrix operations. For two $n\times n$ matrices $\matA,\matB$,  element-wise operations $\{+,-,\otimes,\oslash\}$ satisfy:
\begin{equation}
    \label{eq:elementwise_floating_point_bound}
    \fl(\matA\circ\matB) = \matA\circ\matB + \matDelta \otimes (\matA\circ\matB),\quad |\matDelta_{i,j}|\leq \umach, \quad \|\matDelta \otimes(\matA\circ\matB)\| \leq \umach\sqrt{n}\|\matA\circ\matB\|.
\end{equation}
Moreover, we will use the standard matrix product floating point bound:
\begin{equation}
    \label{eq:matmul_floating_point_bound}
    \fl(\matA\matB) = \matA\matB + \matDelta,\quad \|\matDelta\|\leq \umach\cdot n^C\cdot \|\matA\|\cdot \|\matB\|,
\end{equation}
where $C$ is an absolute constant.

%
%%
%%%
\subsection{Numerical stability}
%%%
%%
%
Fast algorithms are not useful if the returned results are not accurate. In this section, we recall a notion of numerical stability for matrix inversion (see e.g. \cite{croz1992stability,higham2002accuracy}), which allows us to argue about the robustness of different algorithms against numerical errors. It is especially important in low-precision scenarios, which are becoming more and more prominent in AI workloads. We assume a standard floating-point model of computation with $t$ bits for the significand, $p$ bits for the exponent, and one bit for the sign. The machine precision is equal to $\umach:=2^{-t}$.
%
%%
%%%
\begin{definition}
    Let $\matA$ be an invertible $n\times n$ triangular matrix. An algorithm $\mathcal{A}$ which computes an approximate inverse $\matAtilde ^{-1}$ is said to be $f$-stable if:
    \begin{align*}
        \norm{\matA^{-1} - \widetilde{\matA}^{-1}} \leq c_1 \cdot n^{c_2}  \cdot \kappa_2^{f}(\matA) \cdot \umach \cdot \norm{\matA^{-1}},
    \end{align*}
where $c_1$ and $c_2$ are global constants (independent of the matrix size $n$), and $\kappa_2(\matA) = \| \matA \|_2 \, \| \matA^{-1} \|_2 \geq 1$ is the spectral-norm condition number of $\matA$.\label{definition:numerically_stable}
\end{definition}
%%%
%%
%
In the above definition, $f$ can be a constant, or a function of $n$. For more details and different notions of numerical stability we refer to standard textbooks \cite{gallopoulos2016parallelism,golub2013matrix,higham2002accuracy} and references therein. Throughout the paper, we might refer to $O(1)$-stable algorithms as ``fully-stable'' or simply ``stable'', and to $\polylog(n)$-stable algorithms as ``logarithmically-stable'' or ``log-stable'', as coined in \cite{demmel2007fast}.

\subsection{Triangular matrices and their condition number}
One of the challenging and intriguing facts about triangular matrices in general, is that they can easily have an abominable condition number: Viswanath and Trefethen proved that the matrix $\matI+{\tt strict\_tril}(\matG)$, where $\matG_{i,j}\sim \mathcal{N}(0,1)$,  almost surely has an \emph{exponentially growing} condition number \cite{viswanath1998condition}. Since the floating point errors in matrix inversion depend on the condition number of the input, this fact seems very discouraging for inverting matrices in linear transformers, since they are also triangular with unit diagonal elements. 
Luckily, this is not the case. One of the key observations for our analysis is that matrices of the form of Eq. \eqref{equation:definition_of_lower_triangular_matrices} are actually \emph{well-conditioned}. This is necessary to justify the use of low-precision compute units for such operations, and it also reveals why such implementations have been working well in practice to-date.
In \cite{kexuefm-11563}, it was proved that the elements of the inverse of a DeltaNet matrix are always in $[-1,1]$. 

Note that the scaling $\phi(\cdot,\cdot)$ that takes place in Eq. \eqref{equation:definition_of_lower_triangular_matrices} typically reduces the magnitude of the off-diagonal elements, in which case it is expected to further reduce the condition number as the matrix becomes more-and-more diagonally dominant. In that respect, these types of matrices are much better conditioned than random triangular matrices, which is promising towards exploiting low-precision hardware and fast (potentially less stable) algorithms. To prove this formally, let us use the following broad definition, which covers DeltaNet, GDN, and KDA matrices. Appendix \ref{appendix:chunkwise_parallel_form_derivation} outlines the related  background.
\begin{definition}[$\Delta$-matrix]
\label{def:delta_matrix}
We refer to an $n\times n$ matrix $\matA$ as a $\Delta$-matrix if it has the form:
\begin{align*}
    \matA = 
    \matI + \diag(\vecb) 
        \cdot \left(\matGamma \odot \matK \right)
        \cdot \left(\matK \oslash \matGamma\right)^\top
        \odot \matL^{-},
\end{align*}
where the $i$-th row of $\matGamma\in\mathbb{R}^{n\times d}$ is equal to $\veca_0\odot \veca_1\odot\ldots\odot \veca_i$, the $d$-dimensional vectors $\vecb,\veca_0,\veca_1,\ldots,\veca_{n-1}$ have elements in $(0,1]$, and the rows of $\matK\in\mathbb{R}^{n\times d}$ have unit-length. 
\end{definition}

\begin{proposition}
    \label{proposition:condition_number_upper_bound_delta_matrices}
    Let $\matA$ be a $\Delta$-matrix, as in Definition \ref{def:delta_matrix}. $|\matA^{-1}_{i,j}|\leq 1$, for all $i,j=1,\ldots,n$.
    \begin{proof}
        Due to its length, the proof is deferred to Appendix \ref{appendix:proof_of_delta_matrix_inverse_elements}.
    \end{proof}
\end{proposition}
%%%
%%
%

The following implication is immediate.
\begin{corollary}
    \label{corollary:condition_number_upper_bound}
    The condition number of an $n\times n$ $\Delta$-matrix $\matA$ is at most $n^2$.
    \begin{proof}
        It holds that $\max_{i,j} |\matA_{i,j}|\leq 1$, and from Proposition \ref{proposition:condition_number_upper_bound_delta_matrices} we know that $\max_{i,j} |\matA_{i,j}^{-1}|\leq 1$. For any matrix $\matM$, it holds that $\norm{\matM} \leq n\cdot \max_{i,j}|\matM_{i,j}|$, which implies that $\|\matA\|_2\leq n$ and $\norm{\matA^{-1}}\leq n$, and thus $\kappa(\matA)=\norm{\matA}\|\matA^{-1}\|_2\leq n^2$.
    \end{proof}
\end{corollary}

%
%%
%%%
\begin{minipage}{0.28\textwidth}
\begin{algorithm}[H]
\caption{MCS$(\matL,n)$:}
\label{alg:mcs}
\begin{algorithmic}[1]
\footnotesize
\State $\matL \gets 2\matI - \matL$
\State $\matX \gets \matI$
\For{$k = n-1$ down to $0$}
  \State $\matM \gets \matI + \matL{[:,k]}\cdot \vece_k^\top$
  \State $\matX \gets \matM \cdot \matX$ 
\EndFor
\State \Return $\matX$
\end{algorithmic}
\end{algorithm}
\end{minipage}
\begin{minipage}{0.37\textwidth}
\begin{algorithm}[H]
\caption{MXR$(\matL,n,b_0,r=0)$:}
\label{alg:mxr}
\begin{algorithmic}[1]
\footnotesize
\State $\matD\gets$ $b_0\times b_0$ diagonal blocks of $\matL$
\State $\matY \gets $ MCH$(\matD,b_0)$
\For{$i=1,\ldots, r$}
    \State $\matY\gets$ IR$(\matY,\matL)$ \Comment{Alg. \ref{alg:ir}}
\EndFor
\State $\matX \gets $ MBH$(\matL,n,\matY,b_0)$ \Comment{Alg.~\ref{alg:mbh}}
\State \Return $\matX$
\end{algorithmic}
\end{algorithm}
\end{minipage}
\begin{minipage}{0.32\textwidth}
\begin{algorithm}[H]
\caption{NS$(\matL,n,m,\matX_0)$:}
\label{alg:ns}
\begin{algorithmic}[1]
\footnotesize
\State $\matA\gets \matI+\matL$
\State $\matX\gets \matX_0$
\For{$k=0,\ldots,m-1$}
\State $\matY\gets \matA\cdot\matX$
\State $\matX \gets 2\matX - \matX\cdot \matY$
\EndFor
\State \Return $\matX$
\end{algorithmic}
\end{algorithm}
\end{minipage}
\begin{minipage}{0.44\textwidth}
\begin{algorithm}[H]
\caption{MBH$(\matL,n,\matX_0=\matI,b_0=1)$:}\label{alg:mbh}
\begin{algorithmic}[1]
\footnotesize
\State $\matX\gets \matX_0$
\For{$b=b_0, 2b_0, 4b_0,\ldots,n$}
    \State $\matD_{e}\gets \text{even $b\times b$ diagonal blocks of } \matX$
    \State $\matD_{o}\gets \text{odd $b\times b$ diagonal blocks of } \matX$
    \State $
        \matX \leftarrow \matD_{e} + \matD_{o}- \matD_{o}\cdot \matL \cdot \matD_{e}
        $
\EndFor
\State \Return $\matX$
\end{algorithmic}
\end{algorithm}
\end{minipage}
\hfill
\begin{minipage}{0.28\textwidth}
\begin{algorithm}[H]
\caption{MCH$(\matL,n)$:}
\label{alg:mch}
\begin{algorithmic}[1]
\footnotesize
\State $\matX \gets \matI-\matL$
\State $\matY\gets \matL$
\For{$k = 1,2,\ldots,\log(\tfrac{n}{2})$}
  \State $\matY \gets \matY\cdot\matY$
  \State $\matX \gets \matX+\matX\cdot\matY$ 
\EndFor
\State \Return $\matX$
\end{algorithmic}
\end{algorithm}
\end{minipage}
\hfill
\begin{minipage}{0.25\textwidth}
    \vspace{-0.7cm}
    \begin{algorithm}[H]
    \caption{IR$(\matY,\matL)$:}
    \label{alg:ir}
    \begin{algorithmic}[1]
    \footnotesize
    \State $\matA\gets \matI-\matL$
    \State $\matR\gets\matI-\matY\matA$
    \State $\matX\gets \matY+\matR\matY$
    \State \Return $\matX$
    \end{algorithmic}
    \end{algorithm}
\end{minipage}
%%%
%%
%

%
%%
%%%
\section{Analysis of methods}
%%%
%%
%
In this section, we provide the analysis of the main triangular inverse methods considered here, in terms of their stability properties and their computational complexity. Complexity is mainly focused on matrix products, which is also in-line with recent proposals on models of computation that focus on modern AI accelerators  \cite{chowdhury2020computational,chowdhury2021algorithm,sobczyk2025segmented,zouzias2023parallel}. A fine-grained ``bit complexity'' analysis can be inferred by the floating point precision bounds. Algorithms \ref{alg:mcs}-\ref{alg:ir} provide a high-level description of the methods.

\subsection{Vectorized and matrix product-based column-sweep (VCS and MCS)}

The first algorithm that we discuss is the one proposed by Sameh and Brent \cite{sameh1977solving}, which, to our knowledge, was the first parallel algorithm to solve triangular systems with logarithmic depth. It is often referred to as the \emph{column-sweep} method, and it can be implemented in a vectorized form, or in a matrix product-based formulation (see also~\cite{gallopoulos2016parallelism}). 

\begin{lemma}
    \label{lemma:stability_of_column_sweep}
    Column-sweep is numerically-stable, as per Definition \ref{definition:numerically_stable}. For well-conditioned matrices (e.g. $\Delta$-matrices) that satisfy Corollary \ref{corollary:condition_number_upper_bound}, $t=\log(1/\epsilon)+O(\log(n))$ bits of accuracy are sufficient so that the returned solution $\matAtilde^{-1}$ satisfies $\|\matAtilde^{-1}-\matA^{-1}\|_2 \leq \epsilon \|\matA^{-1}\|_2$, for any $\epsilon\in(0,1)$.
    \begin{proof}
        \cite{sameh1977solving} proved that, for any vector $\vecb$, the returned solution $\widetilde \vecx$ for the linear system $\matA\vecx=\vecb$ satisfies $(\matA+\matE)\widetilde\vecx=\vecb$, where $\|\matE\|_{\infty} \leq C\cdot n^2\log(n) \cdot \kappa^2(\matA) \cdot \|\matA\|_{\infty}\cdot \umach$. Since we apply the algorithm for every column of the identity as $\vecb$, we can \cite[Theorem 7.4]{higham2002accuracy}, which implies that $\left\|
        \matAtilde^{-1}-\matA^{-1}
        \right\|_{\infty} 
        \leq 
        \xi\cdot \|\matA^{-1}\|_{\infty},$ where $\xi:=\frac{\|\matA^{-1}\|_{\infty}\|\matE\|_{\infty}}{1-\|\matA^{-1}\|_{\infty}\|\matE\|_{\infty}}$. Using standard norm inequalities we can finally argue that:
        \begin{align*}
            \left\|\widetilde \matA^{-1} - \matA^{-1}\right\|_2
            &\leq \sqrt{n} \left\|
        \matAtilde^{-1}-\matA^{-1}
        \right\|_{\infty} \\
        &\leq 
        \sqrt{n}\frac{\|\matA^{-1}\|_{\infty}\|\matE\|_{\infty}}{1- \|\matA^{-1}\|_{\infty}\|\matE\|_{\infty}} \cdot \|\matA^{-1}\|_{\infty}
        \\
        &\leq 
        C'\cdot\|\matA^{-1}\|_{\infty}\cdot n^{2.5}\log(n) \cdot \kappa^2(\matA) \cdot \|\matA\|_{\infty}\cdot \umach \cdot \|\matA^{-1}\|_{\infty}
        \\
        &\leq 
        C'\cdot n^{3.5}\log(n) \cdot \kappa^3(\matA) \cdot  \umach \cdot \|\matA^{-1}\|_{2},
        \end{align*}
        where the third step requires the assumption that the machine precision is sufficiently small, so that following bound needs to be satisfied: $\|\matA^{-1}\|_{\infty} \|\matE\|_{\infty} \ll 1$.
    \end{proof}
\end{lemma}
The corresponding matrix product-based formulation is described in Algorithm \ref{alg:mcs}. 
In general, the MCS algorithm below generates the matrices $\matM_{n-1} \matM_{n-2} \dots \matM_1$ and performs the chained matrix products. One drawback of the MCS method is that, for matrix size $n$, it requires $n-1$ matrix products. In terms of operations, the complexity of these two algorithms is as follows:
\begin{description}
\item[Complexity (VCS):] $O(n^2)$ vector operations of length $n, n-1, n-2, \ldots, 2, 1$.
\item[Complexity (MCS):] $n-1$ matrix products of size $n \times n$, $O(n)$ vector ``copies'' of length $O(n)$.
\end{description}

\subsection{Inverse via characteristic polynomial (MCH)}

\label{section:unstable_inverse}
The main drawback of the coulmn-sweep algorithm is that it executes $O(n)$ square matrix products of size $n$. 
The authors of~\cite{zhong2025understanding} highlight that an alternative algorithm can be considered for such cases, which reduces inverse computation to $O(\log(n))$ matrix products. It can be straightforwardly derived from \emph{Cayley--Hamilton theorem}, by noting that all eigenvalues are equal to one, and therefore the characteristic polynomial can be derived analytically. The inverse in this case is given by:
\begin{equation}
(\matI+\matL)^{-1} = \matI - \matL + \matL^2 - \matL^3 + \cdots + (-1)^{n-1} \matL^{n-1},
\end{equation}
where $\matL$ is an arbitrary strict triangular square matrix of size $n$, and it can be quickly computed via repeated squaring: we first set $\matX \gets \matI - \matL$ and $\matY \gets \matL$, and, thereafter, for $i = 1, \dots, \lfloor \log(n) \rfloor - 1$, the algorithm updates first $\matY \gets \matY \cdot \matY$, and then $\matX \gets \matX + \matX \cdot \matY$. It finally returns the last $\matX$. An algorithmic description is listed in Algorithm \ref{alg:mch}.
\begin{description}
\item[Complexity (MCH):] $\approx 2\log(n)$ square matrix products of size $n$.
\end{description}
In terms of stability, this algorithm and other ones related to the characteristic polynomial are known to have terrible numerical properties (see e.g. the discussion in \cite{codenotti2001role}). The main issues arise due to the rapid growth of intermediate values caused by the repeated squaring. Consider for example the matrix $\matA=\matI+\matL^{-}$ where $\matL^{-}$ is the strict lower triangular matrix with all entries equal to $-1$. After $\log(k)$ iterations of MCH, the elements of the matrix $\matY_k=(\matL^{-})^k$ can be derived analytically: $(\matL^{-})^k_{i,j} = \binom{i-j-1}{k-1}$. Specifically, the element $(\matL^{-})^{n/2}_{n,1} = \binom{n-2}{n/2-1}$, grows approximately as $\sim 2^n/\sqrt{n}$, up to constant factors.
This can easily lead to overflow, especially in low-precision formats. However, as long as the algorithm is restricted to very small (constant-sized) matrices, the numerical errors can be maintained to an acceptable level. While this algorithm should  be generally avoided, it can be used (carefully) as a subroutine for handling the inverses of small diagonal blocks during the execution of more stable algorithms. 
In the following Lemma we report a bound for the floating point errors of the returned inverse.
\begin{lemma}
    \label{lemma:stability_of_mch}
    Let $\matL\in\mathbb{R}^{n\times n}$ be a strictly lower triangular matrix, and let $\matA:=\matI+\matL$. If the machine precision $\umach$ satisfies $\umach\leq \umach_{MCH}:= \frac{1}{2^n\cdot\mu(n)} $, then Alg. \ref{alg:mch} (MCH) returns a matrix $\matAtilde^{-1}$ such that:
    \begin{align*}
        \left\|\matA^{-1}-\matAtilde^{-1}\right\|_2 \leq \umach \cdot \mu''(n) \cdot \left(
            \phi(\log(\tfrac{n}{2})) + \sum_{j=1}^{\log(\tfrac{n}{2})-1} \phi(j) \cdot \frac{\psi(\log(\tfrac{n}{2}),2,0)}{\psi(\log(\tfrac{n}{2}),2,0)}
        \right), 
    \end{align*}
    where $\mu(n),\mu''(n)$ are low-degree polynomials in $n$, $\psi(k,c,d):=\prod_{i=1}^k (1+(c\cdot \|\matL\|_2)^{2^{i+d}})$, and  $\phi(k):=\left[
        \psi(k,1,0)
        \cdot (2\|\matL\|_2)^{2^k} 
        +
    (1+\|\matL\|_2)  \cdot 2^{2^k} \cdot 
    \psi(k,2,1)
    \cdot 
    \left(
        1 +  (2 \|\matL\|_2)^{2^{k+1}}
    \right)
    \right]$.
    \begin{proof}
        The full proof is deferred to Appendix \ref{appendix:proof_of_stability_of_mch}.
    \end{proof}
\end{lemma}

\subsection{A log-stable matrix-based recursive algorithm (MBH)}
\label{section:log_stable_inverse}
In this section, we revisit a recursive triangular inversion algorithm, which we call MBH, since (to our knowledge) the first analysis is due to Bunch and Hopcroft \cite{bunch1974triangular}. In \cite{demmel2007fast} it was proved that the algorithm is $\polylog(n)$-stable. The algorithm is based on the following observation:
%
%%
%%%
\begin{equation}
\matA^{-1} =
\begin{pmatrix}
\matA_{11} & \matA_{12} \\
0 & \matA_{22}
\end{pmatrix}^{-1}
=
\begin{pmatrix}
\matA^{-1}_{11} & -\matA^{-1}_{11} \matA_{12} \matA_{22}^{-1} \\
0 & \matA_{22}^{-1}
\end{pmatrix}.
\label{eq:recursive_bunch_hopcroft}
\end{equation}
%%%
%%
%
where $\matA_{11}$ and $\matA_{22}$ are square matrices. At each step, it involves computing the inverses of two matrices of half the size, and two matrix products of half the size. In practice, when the individual sub-blocks become very small, it can become a waste of resources to launch a separate matrix product for each tiny block. Instead, it is preferable to use an ``unrolled'' recursion, in order to group together small matrix products that occur at the lowest levels of the recursion. 
The stability analysis of \cite{demmel2007fast}, together with Corollary \ref{corollary:condition_number_upper_bound}, implies a bound on the number of precision bits.
\begin{proposition}
\label{proposition:stability_of_bunch_hopcroft}
Fix $\epsilon\in(0,1)$. It is sufficient to use $t=\log(1/\epsilon)+O(\log^c(n))$ bits of precision to invert a well-conditioned matrix $\matA$ that satisfies Corollary \ref{corollary:condition_number_upper_bound} (e.g. DeltaNet), using Alg.~\ref{alg:mbh}, such that the returned solution $\matAtilde^{-1}$ satisfies $\left\|\matAtilde^{-1}-\matA^{-1}\right\|_2\leq \epsilon\left\|\matA^{-1}\right\|_2$.
\end{proposition}
\begin{description}
\item[Complexity (MBH):] $\approx 2\log(n)$ matrix products of size $n\times n$ (proof in Appendix \ref{appendix:proof_of_lemma_unrolled_mbh_derivation}).
\end{description}
\subsection{A mixed recursive algorithm (MXR)}
\label{section:mixed_recursive_algorithm}
One issue of the recursive algorithm of the previous section is that, for ``tiny'' block-sizes, it needs to execute two full $n\times n$ matrix multiplication per iteration, while most of the elements in these cases are almost zero, since the matrices are almost diagonal. To overcome this for practical purposes, we propose to use the unstable algorithm of Section \ref{section:unstable_inverse} only for a few initial iterations, e.g., for block-sizes up to $16\times 16$, in order to compute the small diagonal inverses as a starting point, and then continue the log-stable iteration for the later steps. This increases the hardware utilization, and, since the unstable algorithm is only executed for constant-sized blocks, the full algorithm in its entirety remains logarithmically stable (albeit, the ``constant errors'' might be large, and should still be treated with care).
An example of this approach, which we later show to be extremely efficient in practice, is outlined below (see also Algorithm \ref{alg:mxr}):
%
%%
%%%
\begin{enumerate}
\item Use the algorithm of Section \ref{section:unstable_inverse} in order to compute the small inverses of the diagonal blocks of $\matA$, where each block has size $C \times C$ for some constant $C$ (e.g., $C\in\{16,32\}$).
\item Use the unrolled recursive algorithm of Section \ref{section:log_stable_inverse} to compute the final inverse.
\end{enumerate}
The stability properties of Algorithm \ref{alg:mxr} are summarized in Theorem \ref{theorem:stability_of_mxr}.
\begin{theorem}
    Algorithm \ref{alg:mxr} is $\polylog(n)$-stable for all input matrices $\matL$ with $\|\matL\|_2=\poly(n)$.
    \begin{proof}
        When $\|\matL\|_2=\poly(n)$, Lemma \ref{lemma:stability_of_mch} implies that the  MCH algorithm, which inverts the small diagonal blocks in the first phase, requires at most $\log(1/\epsilon)+\polylog(n)$ bits of precision to return $\epsilon$-approximate block-inverses, i.e., the first phase is $\polylog(n)$-stable. In the second phase, Algorithm \ref{alg:mbh} takes as input those initial block-inverses to compute the entire inverse, and from \cite{demmel2007fast} we know that this second phase is also $\polylog(n)$-stable.
    \end{proof}
    \label{theorem:stability_of_mxr}
\end{theorem}
\begin{description}
\item[Complexity (MXR):] $\approx 2\log(n)$ matrix products of size $n\times n$.
\end{description}
%%%
%%
%
%
%%
%%%
\subsection{Iterative inversion with Newton-Schulz}
\label{section:newton_schulz}
%%%
%%
%
All aforementioned algorithms so far are \emph{direct}, in a sense that, in exact arithmetic, they return the exact inverse after a finite number of steps. In this section we revisit the Newton-Schulz iteration \cite{schulz1933iterative}, which, as Higham highlights in \cite{higham2002accuracy}, ``...is attractive because it involves only matrix multiplication''. 
Given a square matrix $\matA$, and an initial guess $\matX_0$ for the inverse of $\matA$, the method iteratively updates:
\begin{align}
    \matX_{k+1} \gets \matX_{k}(2\matI-\matA\matX_{k}),\quad k=0,1,2,\ldots
    \label{eq:newton_schulz}
\end{align}

This iteration has been studied in depth for computing the inverse and/or pseudo-inverse of general matrices \cite{pan1985efficient,pan1991improved,soderstrom1974numerical,zakaulah2014efficient}, but we are not aware of a dedicated analysis for unit-diagonal triangular matrices. Nevertheless, we can obtain quite strong convergence and stability guarantees from existing analyses. 
In \cite{pan1985efficient}, Pan and Reif provide explicit bounds on the convergence and number of iterations.

With a straightforward error analysis, one can show that a single iteration of Newton-Schulz is numerically stable. However, carrying over the analysis to multiple iterations to achieve end-to-end numerical stability for the inverse, as in Definition \ref{definition:numerically_stable}, is quite tedious. Shah \cite{shah2025hermitian} recently published end-to-end stability guarantees of the Newton-Schulz iteration for the sign function, and proved that, indeed, the method is stable for that task. Below we provide a similar end-to-end analysis for the Newton-Schulz iteration to compute the matrix inverse.

\begin{lemma}
\label{lemma:stability_of_ns}
    Assume that the Newton-Schulz iteration is implemented in the floating point model as indicated by the parentheses: $\matXtilde_{k+1} \gets \matXtilde_k(2\matI-(\matA\matXtilde_k))$. If the machine precision $\umach$ satisfies $\umach
\leq
\frac{1}{
40n^{C}\kappa^2(\matA)
}$,
for some constant $C>1$, then
after $m=O(\log(n))$ iterations, the returned matrix $\matAtilde^{-1}:=\matXtilde_{m}$ satisfies $
\left\|\matAhat^{-1}-\matAtilde^{-1}\right\|_2\leq \umach \cdot n^{O(1)}\cdot \kappa(\matA) \cdot 
\left\|\matA^{-1}\right\|_2$, where $\matAhat^{-1}$ is the true matrix after $m$ iterations in exact arithmetic.
    \begin{proof}
        Due to its length, the proof is deferred to Appendix \ref{appendix:proof_of_stability_of_ns}.
    \end{proof}
\end{lemma}

\begin{theorem}
    \label{theorem:stability_of_ns}
    Fix $\epsilon>0$. For $\poly(n)$-conditioned triangular matrices, after $O(\log(n))$ iterations $\matAtilde^{-1}$ satisfies $\|\matAtilde^{-1}-\matA^{-1}\|_2\leq \epsilon \|\matA^{-1}\|_2$ using $t=\log(1/\epsilon)+O(\log(n))$ bits of precision, using $\matX_0=\matI/n^{O(1)}$ as an initial guess. 
    \begin{proof}
        The authors in \cite{pan1985efficient} showed that, if $\matA$ is invertible and $\matX_0$ satisfies $\|\matI-\matA\matX_0\|< q:=1-1/n^{O(1)}$, then, in exact arithmetic, after $O(\log(n))$ iterations, the Newton-Schulz method of Eq. \eqref{eq:newton_schulz} returns an approximate inverse
    $\matAhat^{-1}$ which satisfies:
    \begin{align}
        \left\|\matAhat^{-1} - \matA^{-1}\right\|_2\leq  2^{-n^{O(1)}+1}\cdot \left\|\matA^{-1}\right\|_2.
        \label{eq:newton_schulz_convergence_exact_arithmetic}
    \end{align} Now let $\matE:=\matX_0-\matA^{-1}$. We have that $\|\matE\|_2=\|\matA^{-1}\matA\matX_0-\matA^{-1}\|_2\leq\|\matA^{-1}\|_2\cdot \|\matA\matX_0-\matI\|_2\leq q\|\matA^{-1}\|_2<\|\matA^{-1}\|_2$. Thus $\|\matX_0\|_2=\|\matA^{-1}+\matE\|_2\leq (1+q)\|\matA^{-1}\|_2<2\|\matA^{-1}\|_2$.
    From Lemma \ref{lemma:stability_of_ns}, we have that the returned floating point matrix satisfies:
    \begin{align*}
        \|\matAtilde^{-1}-\matAhat^{-1}\|_2\leq \umach\cdot n^{O(1)}\cdot \|\matA^{-1}\|_2,
    \end{align*}
    where we also used the assumption that $\kappa(\matA)=n^{O(1)}$. Then the triangle inequality provides a coarse upper bound:
    \begin{align*}
        \|\matAtilde^{-1}-\matA^{-1}\|_2 &=
            \|\matAtilde^{-1}-\matAhat^{-1}+\matAhat^{-1}-\matA^{-1}\|_2\leq 2\umach\cdot n^{O(1)}\cdot \|\matA^{-1}\|_2
            +
            2^{-n^{O(1)}+1}\cdot \left\|\matA^{-1}\right\|_2.
    \end{align*}
    For reasonable values of $n$ and for $\umach\leq \epsilon/n^{O(1)}$, which translates to $\log(1/\epsilon)+O(\log(n))$ bits of precision, the previous inequality gives $\|\matAtilde^{-1}-\matA^{-1}\|_2\leq \epsilon \|\matA^{-1}\|_2$.

    Note that for triangular matrices that satisfy Corollary \ref{corollary:condition_number_upper_bound}, using $\matX_0=\matI/n^{O(1)}$ is a simple and efficient initial guess that satisfies all the required properties.
    \end{proof}
\end{theorem}
\begin{description}
\item[Complexity (NS):] $2\times N_{iter}$ matrix products of size $n\times n$ (typically $N_{iter}\approx 2\log(n)$).
\end{description}
\subsection{Iterative refinement}
A standard technique to improve the accuracy of numerical algorithms for solving linear systems of equations (especially in floating point arithmetic) is iterative refinement \cite{higham1997iterative,moler1967iterative,wilkinson1948progress}. In the context of this work, we use the following iterative scheme. If $\matAtilde^{-1}_0$ is the approximate inverse returned by an algorithm, then for $k=0,1,\ldots$ we iteratively update:
\begin{align}
    (i)\ \matR_k \gets \matI-\matAtilde_k^{-1}\matA,\qquad (ii)\ \matAtilde^{-1}_{k+1} \gets \matAtilde^{-1}_k + \matR_k\matAtilde^{-1}_k.
    \label{eq:iterative_refinement}
\end{align}
As it is shown later in the numerical experiments, this simple iterative scheme, which involves only two matrix products per iteration, can significantly increase the final accuracy in some cases, even with a single iteration.

\section{Evaluating the numerical accuracy of the methods}
In this section we evaluate experimentally the numerical accuracy of all the aforementioned methods.

\subsection{Accuracy of the computed inverse}
We first assess the accuracy of each method in terms of the posterior errors of the returned inverse compared to the true inverse of the input matrix. For the experimental results illustrated in Figure \ref{fig:inverse_accuracy_comparison_kernel_level}, the input matrix is defined as: 
\begin{align}
    \matA=\matI+\diag(\beta)\cdot {\tt strict\_tril}(\matK\matK^\top),
    \label{eq:matrices_tested_for_numerical_accuracy_of_inverse}
\end{align} 
where each row of $\matK$ is a random vector with independent uniform elements from $[0,1]$, which is then rescaled to have unit length. 

The ``true inverse'' $\matA^{-1}$ is obtained computationally by calling the inverse functions provided by PyTorch in double-precision, which internally calls BLAS/LAPACK routines that are highly accurate. To have a broad overview, we measure three different types of forward errors:
\begin{itemize}
    \item Max absolute: $\max_{i,j}\left|\matA^{-1}_{i,j} - \matAtilde^{-1}_{i,j}\right|,$ 
    \item Max relative: $\max_{i,j\leq i}\frac{\left|\matA^{-1}_{i,j} - \matAtilde^{-1}_{i,j}\right|}{\left|\matA_{i,j}^{-1}\right|},$
    \item Frobenius-norm relative: $\frac{\left\|\matA^{-1}-\matAtilde^{-1}\right\|_F}{\left\|\matA^{-1}\right\|_F}$.
\end{itemize}

We report results for three different types of numerical precision for the inputs: float16, bfloat16, and float32, and two different cases for the vector $\beta$: the all-ones vector, and random vectors with i.i.d elements from $[0,1]$. Internally in the algorithms, matrix products are always executed in float32 precision, even if they receive the input matrices in lower precision; this is standard in modern hardware, since every matrix product might need more bits to store the final result. 
\begin{figure}[htb]
    \centering
    \includegraphics[width=\linewidth]{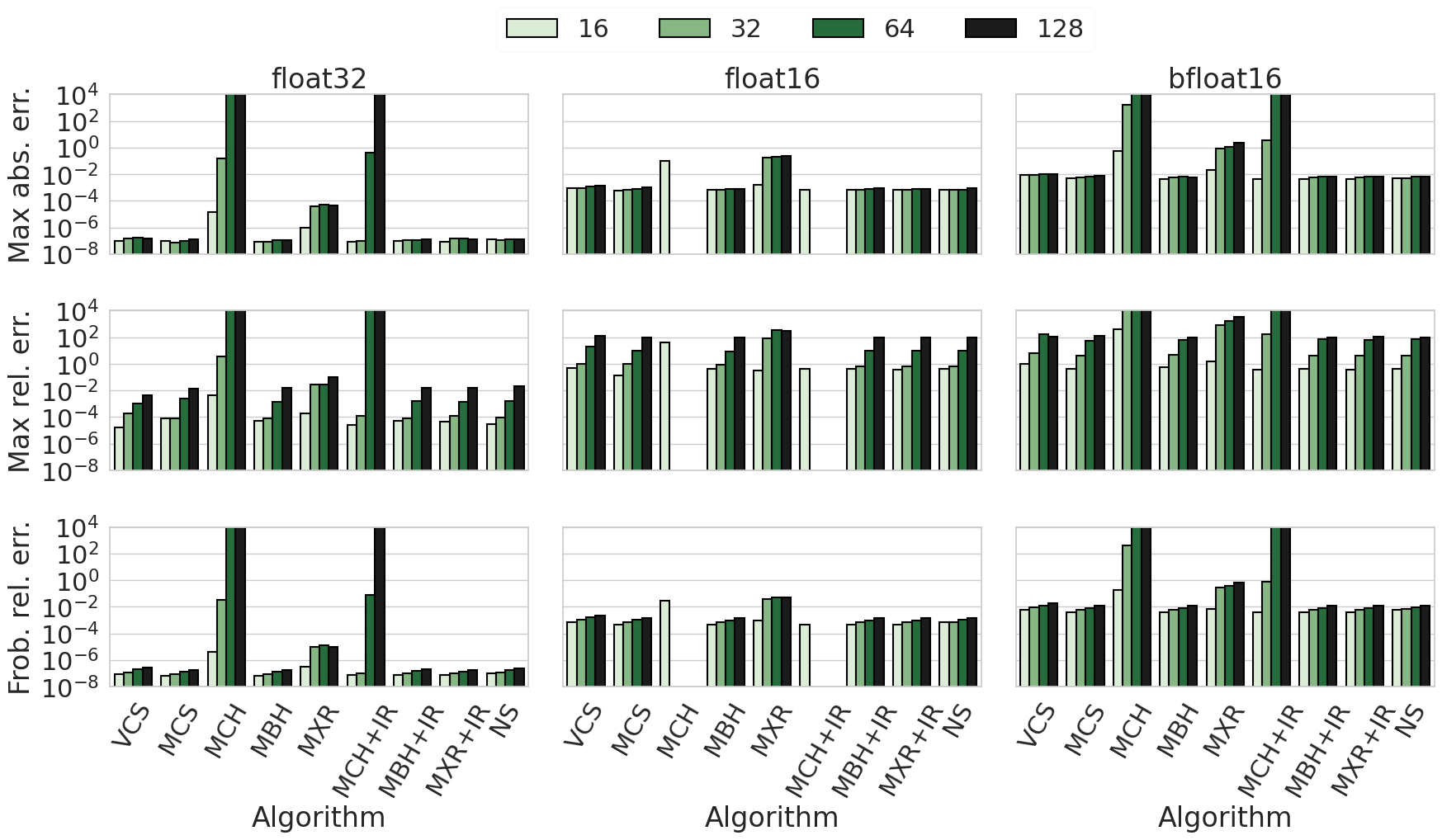}
    \vspace{-0.7cm}
    \caption{Comparison of the inverse accuracy of various methods, for three different input precision formats, and four matrix sizes. The input matrix is sampled as defined in Eq. \eqref{eq:matrices_tested_for_numerical_accuracy_of_inverse}, where $\beta$ is the all-ones vector.}
    \label{fig:inverse_accuracy_comparison_kernel_level}
\end{figure}
In single precision, all methods except MCH and MXR return very accurate results, where the relative errors are close to machine precision. MXR (without refinement) achieves about two to three fewer digits of accuracy than the other methods on average, but this is remediated with a single step of iterative refinement. The unstable algorithm, MCH, returns completely wrong results for sizes 64 and 128, but the errors are within acceptable range for size 16, and barely acceptable for size 32. 

In the half-precision case, float16, the errors of all methods increase substantially. The errors are no longer within the ideal machine precision error range, but they are still within an acceptable threshold, achieving about three-to-four digits of accuracy on average. The unstable algorithm returns NaN values for all sizes above 16, while, again, MXR achieves decent accuracy, which is on-par with the other methods after a single step of iterative refinement.

The bfloat16 format slightly ``improves'' the accuracy of the unstable MCH algorithm: the returned values are no longer NaN, which is expected since bfloat16 has a larger exponent range than float16, and can help with extremely large values. However, the errors of MCH remain unacceptably large, and all the returned digits are completely wrong. At the same time, all other methods return higher errors in bfloat16 than in float16, since bfloat16 uses fewer bits for the mantissa. With these observations, float16 should be preferred against bfloat16 for inverting triangular matrices in linear transformers architectures.

In Figure \ref{fig:inverse_accuracy_comparison_kernel_level_decay} we provide the same experiment as in Figure \ref{fig:inverse_accuracy_comparison_kernel_level}, where the strictly lower-triangular part of the input matrix is additionally scaled by a decay factor. In particular, we set the vector $\beta$ to have random elements in $[0,1]$. As expected, this scaling benefits the accuracy of all algorithms. MXR, in particular, gains one to two digits of accuracy for each data-type and matrix size. However, the final accuracy is still unacceptable.
\begin{figure}[htb]
    \centering
    \includegraphics[width=\linewidth]{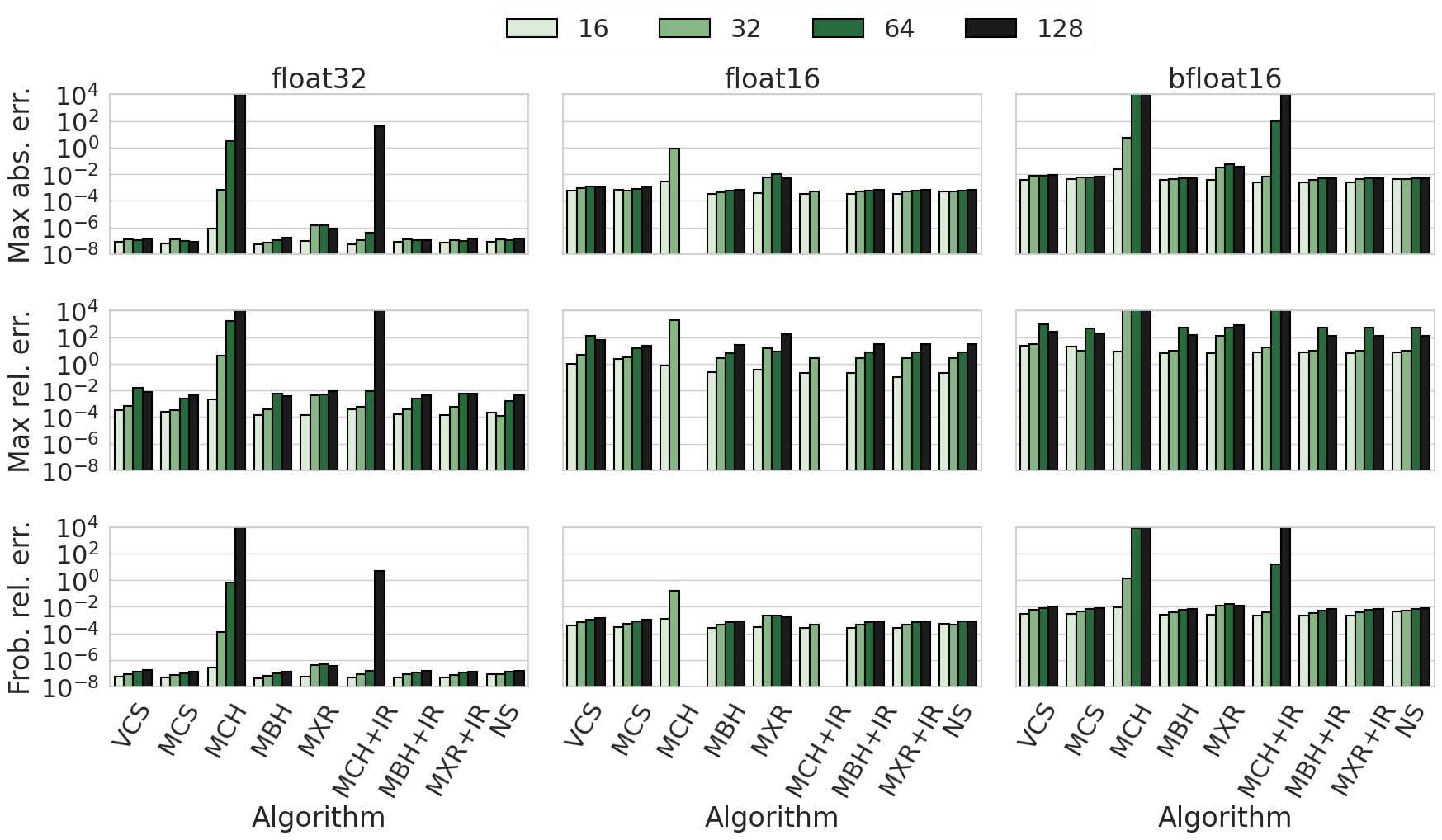}
    \caption{Comparison of the inverse accuracy of various methods, for three different input precision formats, and four matrix sizes. The input matrix is sampled as defined in Eq. \eqref{eq:matrices_tested_for_numerical_accuracy_of_inverse}, where $\beta$ is a vector with uniform random elements in $[0,1]$.}
    \label{fig:inverse_accuracy_comparison_kernel_level_decay}
\end{figure}

%
%%
%%%
\subsection{The effect of numerical accuracy in end-to-end benchmarks}
%%%
%%
%
As already indicated in Fig.~\ref{fig:gdn_runtime_breakdown} (right), an unstable algorithm such as Alg. \ref{alg:mch} can easily collapse the model accuracy to zero. In the Appendix (Fig.~\ref{fig:inverse_accuracy_comparison_extended}) we show that all the stable algorithms, including Alg. \ref{alg:mxr} (MXR) which uses Alg. \ref{alg:mch} as a subroutine, maintain end-to-end accuracy for three different models. The end-to-end accuracy is obtained by modifying the current implementation of the inverse on SGLang~\cite{zheng2024sglang} to use our algorithms, and then by running the MMLU~\cite{hendryckstest2021} benchmark over 10 different subjects. The optimized kernels of SGLang that we use as baseline can be found online \cite{sglang-kernel-npu} and such baseline is labeled in Fig.~\ref{fig:gdn_runtime_breakdown} (right) as ``Triton''. We keep the same naming throughout the whole paper.

Massive Multitask Language Understanding (MMLU)~\cite{hendryckstest2021} is a widely used benchmark for evaluating large language models. It consists of thousands of multiple-choice questions spanning 57 subjects. 
Figure \ref{fig:inverse_accuracy_comparison_extended} illustrates the average accuracy obtained on the MMLU benchmark, for 10 different subjects, for three different models with different inverse algorithms. All algorithms in Figure \ref{fig:inverse_accuracy_comparison_extended} are on-par with the baseline accuracy. This figure complements Figure \ref{fig:inverse_accuracy_comparison_kernel_level}, which already show-cased that the unstable algorithm MCH can lead to catastrophic accuracy drop, when used naively.
\begin{figure}[htb]
    \centering
    \includegraphics[width=0.7\linewidth]{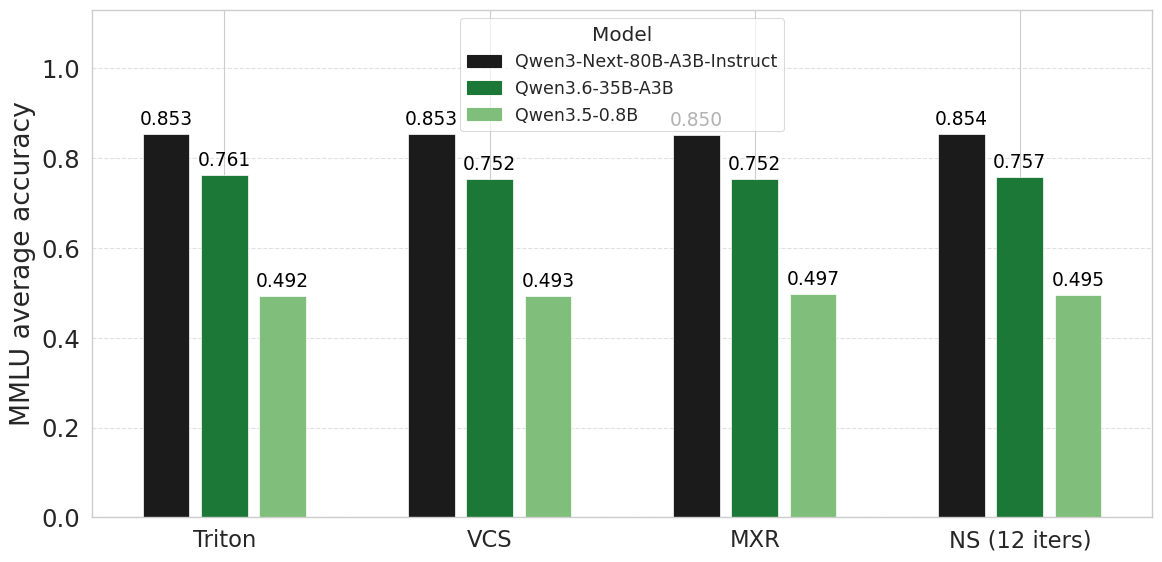}
    \caption{Comparison of the accuracy obtained on MMLU (10-subjects) of different models with different inverse algorithms.}
    \label{fig:inverse_accuracy_comparison_extended}
\end{figure}

We next investigate numerical accuracy in more detail by analyzing how errors in computing the triangular inverse affect end-to-end model performance. We employ the Newton–Schulz method and evaluate it at different iteration counts. For each checkpoint, we measure the relative error with respect to the double-precision inverse computed using NumPy, as well as the average accuracy of several models on MMLU.

\begin{figure}[htb]
\centering
    \centering
    \includegraphics[width=0.55\linewidth]{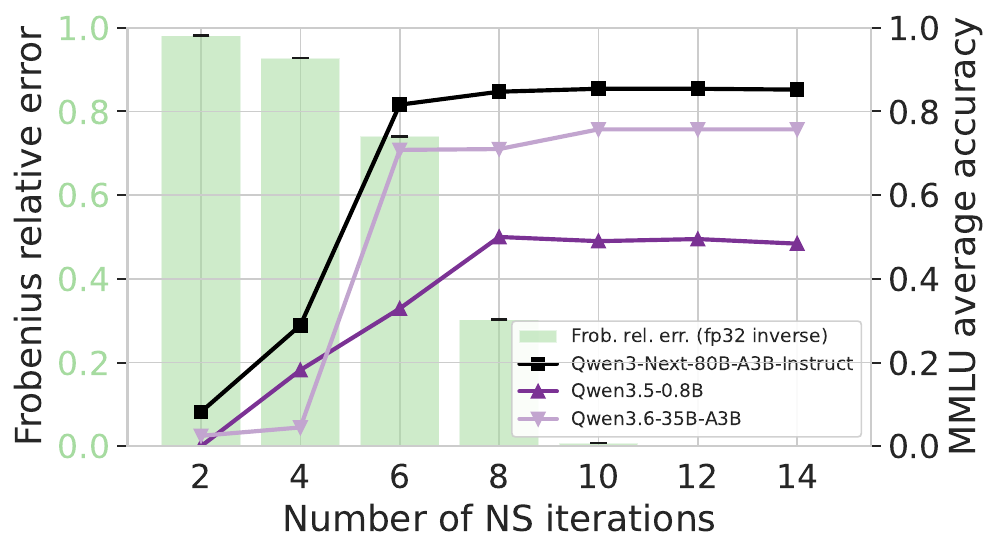}
    \caption{Number of Newton-Schulz iterations (x-axis) versus matrix inversion error using fp32 (left y-axis) and MMLU average accuracy (right y-axis). Chunk size $C=64$.}
    \label{fig:ns-iterations}
\end{figure}

Interestingly, as seen in Fig.~\ref{fig:ns-iterations}, even after only six iterations, the model already starts to perform relatively well, with 81.6\% accuracy for Qwen3-Next-80B-A3B-Instruct. While this is an impressive result considering the relative errors, the accuracy drops are still significant: 4.9\% (Qwen3.6-35B-A3B), 16.6\% (Qwen3.5-0.8B), and 3.8\% (Qwen3-Next-80B-A3B-Instruct) compared to the 12-iteration setting, which ensures numerical stability. 
We further observe that increasing the number of iterations beyond the optimal value (e.g., 14 iterations) does not have any noticeable effect on end-to-end performance.

In Tables \ref{tab:mmlu_inv_methods} and 
\ref{tab:mmlu_ns_iters} we report the full results of the MMLU benchmarks summarized in Fig. 
\ref{fig:inverse_accuracy_comparison_extended} and Fig. 
\ref{fig:ns-iterations} respectively.

\begin{table}[htb]
\centering
\setlength{\tabcolsep}{3pt}
\footnotesize
\caption{MMLU accuracy (10-subject subset) across inversion methods and Qwen models.}
\label{tab:mmlu_inv_methods}
\begin{tabular}{llrrrrrrrrrrr}
\toprule
\textbf{Model} & \textbf{Method} & \textbf{Avg.} & \rotatebox{75}{\strut AbsAlg} & \rotatebox{75}{\strut Anatomy} & \rotatebox{75}{\strut Astron} & \rotatebox{75}{\strut BusEth} & \rotatebox{75}{\strut ClinKn} & \rotatebox{75}{\strut ColBio} & \rotatebox{75}{\strut ColChem} & \rotatebox{75}{\strut ColCS} & \rotatebox{75}{\strut ColMath} & \rotatebox{75}{\strut ColMed} \\
\midrule
\multirow{4}{*}{\textbf{Qwen3-80B-A3B}} & Triton & 0.853 & 0.800 & 0.807 & 0.934 & 0.810 & 0.894 & 0.972 & 0.670 & 0.850 & 0.780 & 0.861 \\
 & VCS & 0.853 & 0.800 & 0.815 & 0.934 & 0.810 & 0.894 & 0.972 & 0.670 & 0.840 & 0.770 & 0.867 \\
 & MXR & 0.850 & 0.800 & 0.800 & 0.934 & 0.800 & 0.894 & 0.965 & 0.670 & 0.850 & 0.760 & 0.861 \\
 & NS-12 & 0.854 & 0.800 & 0.807 & 0.941 & 0.810 & 0.891 & 0.972 & 0.680 & 0.860 & 0.770 & 0.861 \\
\midrule
\multirow{4}{*}{\textbf{Qwen3-35B-A3B}} & Triton & 0.761 & 0.480 & 0.830 & 0.737 & 0.850 & 0.823 & 0.944 & 0.470 & 0.770 & 0.680 & 0.803 \\
 & VCS & 0.752 & 0.460 & 0.830 & 0.737 & 0.830 & 0.811 & 0.931 & 0.460 & 0.760 & 0.700 & 0.780 \\
 & MXR & 0.752 & 0.500 & 0.830 & 0.717 & 0.830 & 0.819 & 0.951 & 0.440 & 0.750 & 0.680 & 0.775 \\
 & NS-12 & 0.757 & 0.460 & 0.837 & 0.743 & 0.840 & 0.826 & 0.938 & 0.450 & 0.790 & 0.680 & 0.775 \\
\midrule
\multirow{4}{*}{\textbf{Qwen3.5-0.8B}} & Triton & 0.492 & 0.320 & 0.467 & 0.500 & 0.510 & 0.581 & 0.618 & 0.470 & 0.420 & 0.350 & 0.491 \\
 & VCS & 0.493 & 0.330 & 0.467 & 0.507 & 0.500 & 0.581 & 0.604 & 0.490 & 0.420 & 0.330 & 0.503 \\
 & MXR & 0.497 & 0.330 & 0.452 & 0.507 & 0.540 & 0.600 & 0.590 & 0.490 & 0.440 & 0.310 & 0.503 \\
 & NS-12 & 0.495 & 0.330 & 0.467 & 0.493 & 0.530 & 0.589 & 0.604 & 0.490 & 0.430 & 0.330 & 0.491 \\
\bottomrule
\end{tabular}
\end{table}

\begin{table}[htb]
\centering
\setlength{\tabcolsep}{3pt}
\footnotesize
\caption{MMLU accuracy (10-subjects) across Newton-Schulz iteration counts and Qwen models.}
\label{tab:mmlu_ns_iters}
\begin{tabular}{llrrrrrrrrrrr}
\toprule
\textbf{Model} & \textbf{Method} & \textbf{Avg.} & \rotatebox{75}{\strut AbsAlg} & \rotatebox{75}{\strut Anatomy} & \rotatebox{75}{\strut Astron} & \rotatebox{75}{\strut BusEth} & \rotatebox{75}{\strut ClinKn} & \rotatebox{75}{\strut ColBio} & \rotatebox{75}{\strut ColChem} & \rotatebox{75}{\strut ColCS} & \rotatebox{75}{\strut ColMath} & \rotatebox{75}{\strut ColMed} \\
\midrule
\multirow{7}{*}{\textbf{Qwen3-80B-A3B}} & NS-2 & 0.082 & 0.130 & 0.044 & 0.059 & 0.130 & 0.060 & 0.076 & 0.030 & 0.060 & 0.130 & 0.127 \\
 & NS-4 & 0.289 & 0.250 & 0.400 & 0.263 & 0.270 & 0.302 & 0.410 & 0.110 & 0.290 & 0.240 & 0.272 \\
 & NS-6 & 0.816 & 0.680 & 0.807 & 0.941 & 0.820 & 0.864 & 0.938 & 0.640 & 0.770 & 0.680 & 0.821 \\
 & NS-8 & 0.847 & 0.780 & 0.815 & 0.934 & 0.820 & 0.883 & 0.965 & 0.670 & 0.850 & 0.760 & 0.850 \\
 & NS-10 & 0.854 & 0.800 & 0.807 & 0.941 & 0.810 & 0.891 & 0.965 & 0.670 & 0.860 & 0.780 & 0.867 \\
 & NS-12 & 0.854 & 0.800 & 0.807 & 0.941 & 0.810 & 0.891 & 0.972 & 0.680 & 0.860 & 0.770 & 0.861 \\
 & NS-14 & 0.852 & 0.800 & 0.807 & 0.934 & 0.810 & 0.891 & 0.965 & 0.670 & 0.850 & 0.780 & 0.861 \\
\midrule
\multirow{7}{*}{\textbf{Qwen3-35B-A3B}} & NS-2 & 0.025 & 0.000 & 0.104 & 0.013 & 0.030 & 0.019 & 0.014 & 0.020 & 0.000 & 0.000 & 0.035 \\
 & NS-4 & 0.045 & 0.130 & 0.074 & 0.092 & 0.000 & 0.030 & 0.014 & 0.010 & 0.000 & 0.000 & 0.075 \\
 & NS-6 & 0.708 & 0.360 & 0.733 & 0.868 & 0.750 & 0.770 & 0.882 & 0.460 & 0.710 & 0.530 & 0.728 \\
 & NS-8 & 0.710 & 0.410 & 0.704 & 0.737 & 0.830 & 0.766 & 0.917 & 0.420 & 0.790 & 0.610 & 0.717 \\
 & NS-10 & 0.757 & 0.480 & 0.837 & 0.750 & 0.830 & 0.823 & 0.944 & 0.460 & 0.760 & 0.660 & 0.792 \\
 & NS-12 & 0.757 & 0.460 & 0.837 & 0.743 & 0.840 & 0.826 & 0.938 & 0.450 & 0.790 & 0.680 & 0.775 \\
 & NS-14 & 0.757 & 0.490 & 0.837 & 0.743 & 0.840 & 0.819 & 0.931 & 0.440 & 0.770 & 0.690 & 0.792 \\
\midrule
\multirow{7}{*}{\textbf{Qwen3.5-0.8B}} & NS-2 & 0.000 & 0.000 & 0.000 & 0.000 & 0.000 & 0.000 & 0.000 & 0.000 & 0.000 & 0.000 & 0.000 \\
 & NS-4 & 0.182 & 0.170 & 0.333 & 0.211 & 0.170 & 0.155 & 0.236 & 0.080 & 0.170 & 0.030 & 0.202 \\
 & NS-6 & 0.329 & 0.260 & 0.363 & 0.375 & 0.320 & 0.377 & 0.354 & 0.250 & 0.330 & 0.230 & 0.312 \\
 & NS-8 & 0.500 & 0.320 & 0.467 & 0.507 & 0.550 & 0.600 & 0.576 & 0.480 & 0.410 & 0.350 & 0.526 \\
 & NS-10 & 0.490 & 0.320 & 0.474 & 0.487 & 0.510 & 0.574 & 0.611 & 0.490 & 0.420 & 0.340 & 0.491 \\
 & NS-12 & 0.495 & 0.330 & 0.467 & 0.493 & 0.530 & 0.589 & 0.604 & 0.490 & 0.430 & 0.330 & 0.491 \\
 & NS-14 & 0.484 & 0.310 & 0.452 & 0.487 & 0.500 & 0.577 & 0.604 & 0.490 & 0.410 & 0.320 & 0.486 \\
\bottomrule
\end{tabular}
\end{table}

\section{Performance evaluation and optimization strategies}
In this section we report some insights on the performance of the studied algorithms on NPUs, which have recently received significant attention especially for inference tasks. The goal is to understand to which extent the proposed algorithms can efficiently utilize the underlying hardware of modern AI accelerators, and how competitive they are with well-established implementations (Triton, TileLang~\cite{wang2026tilelang}) on major LLM-serving frameworks.

The NPU devices used in our experimentation were Ascend AI accelerators (910B2), containing $48$ SIMD units and $24$ matrix multiplication cores capable of performing square matrix multiplication of size up to $128$. The theoretical memory bandwidth
is 1.2 TB/s. The host is an ARM-based CPU with 48 cores.
Major LLM serving frameworks, in particular, SGLang and vLLM, officially support NPUs and provide optimized inverse primitives that we use as baselines.
\begin{figure}[htb]
    \centering
    \includegraphics[width=\textwidth]{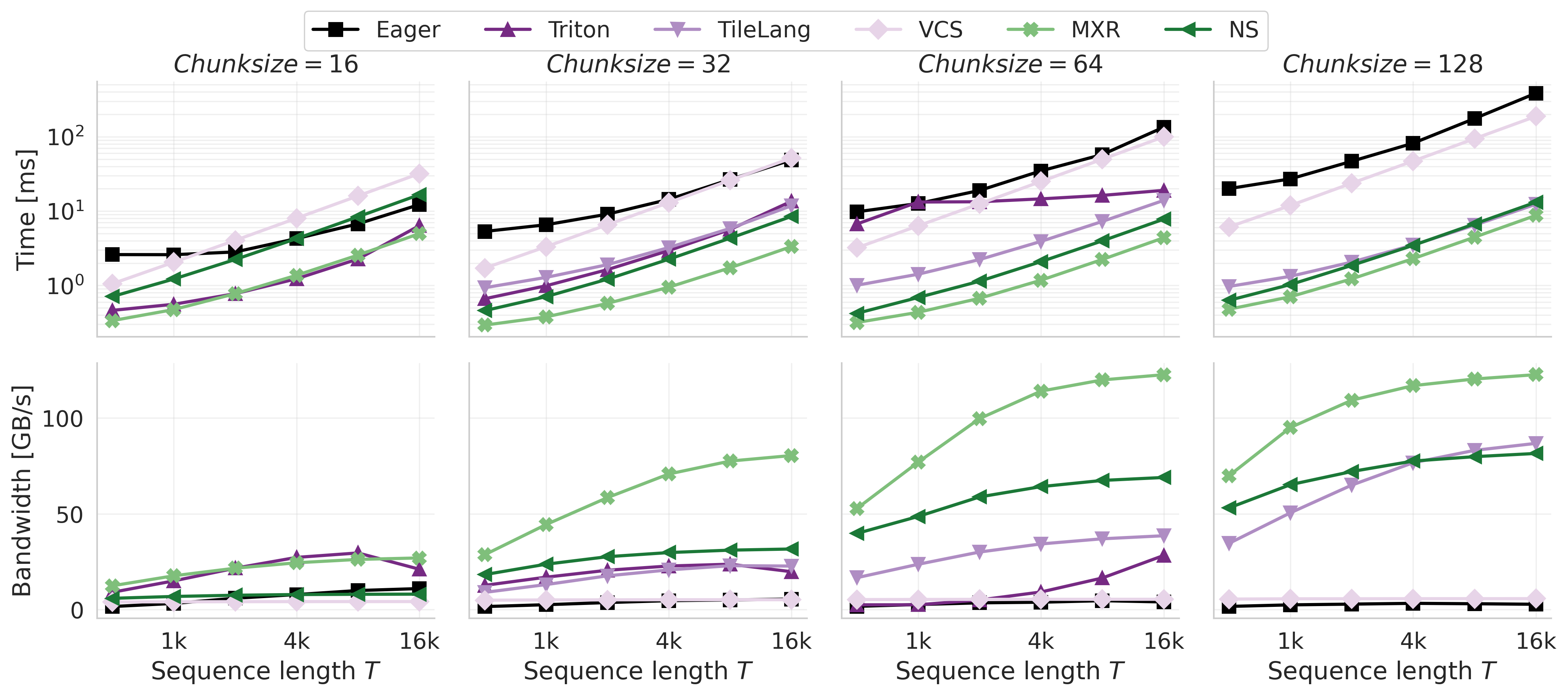}
    \caption{Kernel evaluation of different inverse algorithms. $B=32$, $N=4$ and varying sequence length up to $16K$.}
    \label{fig:kernel-eval}
\end{figure}

\subsection{Parallelization and asynchronous execution}
To achieve the best possible performance, we highlight some high-level practices. First, note that, in existing implementations of linear transformers architectures, the triangular inversion part is inherently massively parallel. Typically, the input is given as a high-dimensional tensor with numerous small-sized chunks, which are all completely independent, and can be inverted without communication. However, internally, within the inversion algorithms, there are data dependencies between each iteration (or recursion level, for the recursive ones). The effect of these dependencies can be mitigated to a minimal level by pipelining.

As an example, in the repeated-squaring algorithm (MCH), Step 1. depends on the corresponding $\matY$ matrix from the previous iteration, while Step 2. expects as an input the output of Step 1. However, the update of $\matY$ in iteration $k+1$ is independent of the update of $\matX$ in iteration $k$, and therefore these two operations can be overlapped, if enough matrix cores are available. Similar observations can be made for the other algorithms as well.
At the ``I/O level'', since individual chunks are independent, it is possible to overlap the inversion of a chunk while simultaneously loading the next chunk from the main (slow) memory to the local memory.

\subsection{Scalability of individual inverse kernels}
At the kernel level, we measure the execution time of the algorithms in isolation.
We run 5 warm-up iterations followed by 20 measured iterations, and clear the data cache before each iteration.
We set batch size $32$, $4$ attention heads, and varying sequence length up to $16K$.
In Fig.~\ref{fig:kernel-eval}, we compare our algorithms against three baselines.
The Pytorch Eager reference is a basic python implementation based on the column-sweep algorithm. The Triton implementation is the official implementation for SGLang on NPUs~\cite{sglang-kernel-npu} and uses a variation of the MBH algorithm. The Tilelang implementation is another kernel that is provided by the Tilelang framework for NPUs~\cite{tilelang-ascend} as an example of optimized algorithm. They also used a similar approach to Triton's, but in a different framework, and by materializing some dimensions of the input.

Our best algorithm, MXR, achieves up to \textbf{4.3$\times$} speedup over the official implementation and a \textbf{$3\times$} speedup over TileLang. This confirms that MXR is the best among the proposed candidates as a fast and stable triangular inverse algorithm for GDNs on NPUs.
\subsection{Full-layer performance comparison}

To see how our kernels perform at the layer level we embed our best performing kernel, MXR, into the GDN layer implemented in~\cite{sglang-kernel-npu}. 
As it is well-known, however, improvements at the kernel level can easily get ``absorbed'' in end-to-end benchmarks by other, unoptimized kernels, due to Amdahl's law. To that end, we further improved the GDN layer by optimizing the other five kernels of the GDN layer, and combining them into a ``MegaKernel'', which also supports chunk size 128. 

\begin{figure}[htb]
    \centering
    \includegraphics[width=0.7\textwidth]{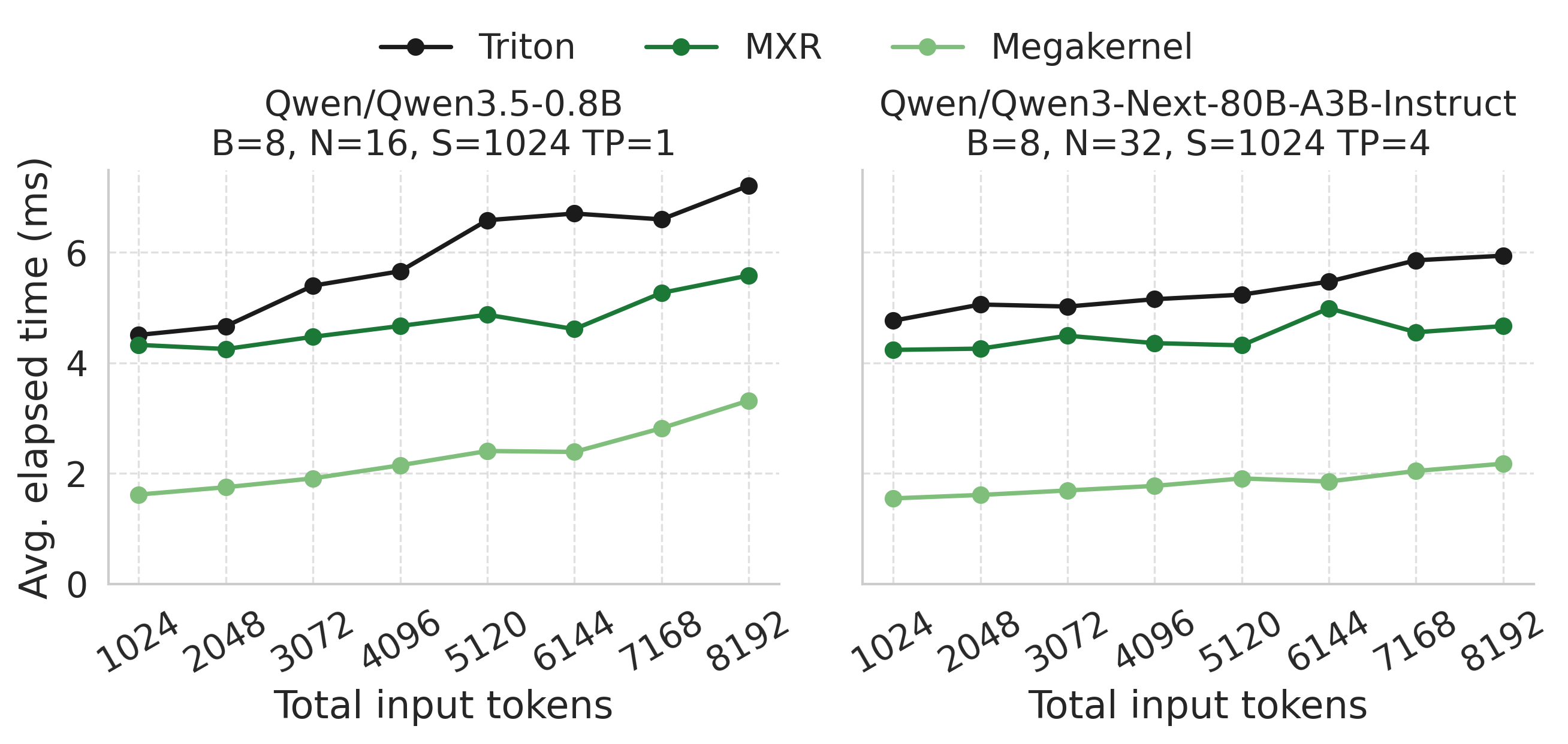}
    \caption{GDN layer runtime on device using two common configurations for NPUs}
    \label{fig:layer-eval}
\end{figure}

In Fig.~\ref{fig:layer-eval} we evaluate the performance of these two approaches and compare it to the existing implementation (Triton). We provide the results for two common configurations of the GDN layer when running Qwen models on NPUs: a small model run on one NPU (Qwen3.5-0.8B) and a big model run on 4 NPUs with tensor parallelism (Qwen3-Next-80B). The figure illustrates the average runtime as the input size (number of tokens) increases. In the small model configuration we obtain up to $1.29\times$ speedup with the inverse kernel substitution and $2.18\times$ speedup with the MegaKernel. With the big model configuration we obtain up to $1.27\times$ speedup with the kernel substitution and $2.74\times$ speedup with the MegaKernel.

\subsection{End-to-end inference performance evaluation}
In this section we provide some end-to-end performance results, regarding the effect of our implementations in the entire inference level. 
We evaluated the following two models served via the \texttt{sglang} framework: Qwen/Qwen3.5-0.8B-Base (served on 1 NPU), Qwen/Qwen3.6-35B-A3B (served on 2 NPUs).
We used the \texttt{sglang.bench\_one\_batch} functionality, which serves the model and executes queries to report latency and throughput during the prefill and decode phase. Three different variations were tested for each model:
\begin{itemize}
    \item default: The default GDN backend of \texttt{sglang}.
    \item MXR: As above, but replacing the \texttt{solve\_tril} method with our implementation.
    \item MegaKernel: Replacing the entire GDN layer with our MegaKernel.
\end{itemize}

In the following table we report the median as well as the 25th and  75th percentile of the prefill throughput in tokens-per-second (TPS, the higher values the better) that was measured by \texttt{sglang}.
\begin{table}[htb]
\centering
\begin{tabular}{l l r r r}
\hline
model &	backend & median TPS & 25\% TPS & 75\% TPS \\
\hline\hline
Qwen3.5-0.8B-Base &default  & 18169 & 18012 & 18344 \\
Qwen3.5-0.8B-Base &MXR & 18924 & 18675 & 19060 \\
Qwen3.5-0.8B-Base &MegaKernel & 20216 & 20045 & 20700\\\hline
Qwen3.6-35B-A3B &default & 8790 & 8691 & 8849\\
Qwen3.6-35B-A3B &MXR & 8906 & 8837 & 9026\\
Qwen3.6-35B-A3B &MegaKernel & 9814 & 9434 & 9879 \\\hline
\end{tabular}
\end{table}

For the Qwen3.5-0.8B-Base model, MXR achieves a $\approx 4\%$ prefill median throughput increase, while the MegaKernel achieves a $\approx11\%$ increase. For the Qwen3.6-35B-A3B model, the MXR backend achieves a smaller increase (about $\approx 1.3\%$, but still outside the statistical error), while the MegaKernel demonstrates again around $\approx 11\%$ increase in the median prefill throughput.

\section{Conclusion}
In this work we revisited triangular inversion algorithms in the context of Linear Attention with the Delta-Rule. We studied both direct and iterative methods, focusing on algorithms that are rich in matrix products and have the potential to efficiently utilize modern hardware that supports such operations. For each method, a detailed numerical stability and complexity analysis was provided, highlighting their effectiveness and their robustness against numerical errors. Besides the theoretical properties, we provided detailed experimental evaluations, covering a wide-spectrum of properties. These included numerical accuracy benchmarks, both at the kernel-level, in different floating point formats, and at the end-to-end model accuracy level. Moreover, we implemented and empirically evaluated the performance of the methods on NPUs. At the kernel level, our fastest and stable implementation achieved up to $4.3\times$ speedup against optimized baselines, and up to $2.5\times$ speedup at the entire model-layer level, when optimized as a whole. We hope that these findings will serve as guidance for fast and robust AI kernel development, with high performance, and, importantly, reliable numerical results.

\bibliographystyle{plain}

%%%%%%%%%%
% APPENDIX
%%%%%%%%%%

% \newpage
\appendix

\section{Chunkwise parallel form and triangular inverse derivation}
\label{appendix:chunkwise_parallel_form_derivation}
Here we provide more details for the chunkwise parallel form. As mentioned in the introduction, the inputs are three matrices $\matQ,\matK,\matV\in\mathbb{R}^{L\times d}$, where $L$ is the sequence length and the rows of $\matK$ are normalized to have unit length. In linear attention, after preprocessing the matrices $\matQ$ and $\matK$ with some suitable functions (e.g. ELU), the $t$-th output, for $t=1,\ldots, L$, is given by Eq. \eqref{eq:linear_attention_recurrent_form}, which can be also written in the matrix form of Eq. \eqref{eq:linear_attention_matrix_form}, which we repeat here for readability:
%
%%
%%%
\begin{align*}
    \matO&=(\matQ\matK^\top \odot\matM)\matV\quad \text{(matrix form)}, \\
    \veco_t&=\left(\sum_{j=1}^{t}\vecv_j\veck_j^\top \right)\vecq_t= \matV^\top_{1:t,:}\matK_{1:t,:}\vecq_t\quad \text{(linear recurrence)},
\end{align*}
%%%
%%
%

The matrix $\matS_t = \matV_{1:t, :}^\top\matK_{1:t,:}$ is often referred to as the \emph{state matrix}, and the linear recurrence formulation can be rewritten as a state update:
%
%%
%%%
\begin{align}
    \veco_t = \matS_{t-1}\vecq_t + \vecv_t\veck_t^\top \vecq_t.
\end{align}
%%%
%%
%

\paragraph{Chunkwise processing.} When the processing is performed chunkwise, the sequence of length $L$ is partitioned into $\lceil L/C\rceil$ ``chunks'' of length $C$ each. This gives rise to individual query-key-value matrices $\matXi_{[k]}:=
    \matXi_{kC:(k+1)C,:}\in\mathbb{R}^{C\times d}$, for $k=0,\ldots,\lceil L/C\rceil-1,$
where $\matXi$ can be any of $\{\matQ,\matK,\matV\}$. The corresponding state matrix is defined as:
\begin{align}
    \matS_{[k]}&:=\matV^\top_{[k]}\matK_{[k]}\in\mathbb{R}^{d\times d}.
\end{align}
To simplify the notation, hereafter we focus on a single chunk.

\paragraph{Delta rule and gating.} The ``Delta rule'' \cite{schlag2021linear} generalizes the state update definition by introducing a scalar $\beta_t$ (a learning rate):
%
%%
%%%
\begin{align}
    \matS_t^{\text{DeltaNet}} &= \matS_{t-1} + \beta_t(\vecv_t-\matS_{t-1}\veck_t)\veck_t^\top  \nonumber\\ 
    &= \matS_{t-1}\left(\matI-\beta_t \veck_t\veck_t^\top\right) + \beta_t\vecv_t\veck_t^\top.
\end{align}
%%%
%%
%
For the rest of our analysis, we use the standard assumption that $\beta_t\in(0,1)$. We highlight, however, that $\beta_t\in(0,2)$ has been carefully studied in the literature and provides novel insights, such as state-tracking capabilities \cite{grazzi2025unlocking,siems2026deltaproduct}. Our analysis can be adapted to also cover this case, as well as the DeltaProduct generalization introduced in \cite{siems2026deltaproduct}, provided the underlying key vectors are unit length (which is indeed the case in \cite{siems2026deltaproduct}).

GDNs further generalize DeltaNets by introducing a  weight-decay factor $a_t\in[0,1]$:
\begin{align}
    \matS_t^{\text{GDN}} = a_t\matS_{t-1}(\matI-\beta_t\veck_t\veck_t^\top)+ \beta_t\vecv_t\veck_t^\top.
\end{align}
This scalar decay can be further extended to a diagonal matrix form, as introduced in Kimi Delta Attention (KDA) \cite{kimi2025kimi}: 
\begin{align}
    \matS_t^{\text{KDA}} = \matS_{t-1}\diag(\veca_t) (\matI-\beta_t\veck_t\veck_t^\top)  + \beta_t\vecv_t\veck_t^\top.
    \label{eq:kda_state_matrix_update}
\end{align}
The KDA update is the most general one, since it covers GDN when all the elements of $\veca_t$ are the same, and also DeltaNet reduces to it when $\diag(\veca_t)$ is the identity matrix. We will therefore focus on the KDA update for the rest of this section.

\paragraph{Householder matrices and triangular inverse.}
By expanding the state update formulation, we have that:
\begin{align}
    \matS_t^{\text{KDA}} = \matS_{0} \prod_{i=1}^t\diag(\veca_i) (\matI-\beta_i\veck_i\veck_i^\top)  + 
    \sum_{i=1}^t \beta_i\vecv_i\veck_i^\top \prod_{j=i+1}^t \diag(\veca_j)(\matI-\beta_j\veck_j\veck_j^\top),
    \label{eq:kda_state_matrix_update_expanded}
\end{align}
where $\prod_{i=1}^t\diag(\veca_i) (\matI-\beta_i\veck_i\veck_i^\top):=\diag(\veca_1) (\matI-\beta_1\veck_1\veck_1^\top) \cdot \diag(\veca_2) (\matI-\beta_2\veck_2\veck_2^\top)\ldots$, i.e., the non-commutative product is defined in ascending order left-to-right.

In \cite{kimi2025kimi,yang2023gated,yang2024parallelizing}, the authors used classic tools from Numerical Linear Algebra to reformulate the state update in a parallelizable, matrix-multiplication friendly manner. Specifically, the main techniques used trace back to the so-called WY and UT transforms, which provide a way to factorize products of generalized Householder transformations \cite{bischof1987wy,joffrain2006accumulating}. Briefly, the UT transform states that the product of Householder matrices admits a factorization of the form:
\begin{align*}
    (\matI-\beta_1\vecu_1\vecu_1^\top)\cdot
    (\matI-\beta_2\vecu_2\vecu_2^\top)\cdot
    \ldots\cdot
    (\matI-\beta_k\vecu_k\vecu_k^\top)
    =
    \matI-\matU_k\matT_k^{-1}\matU_k^\top,
\end{align*}
where $\matU=\begin{pmatrix}
    \vecu_1 & \vecu_2 & \ldots & \vecu_k
\end{pmatrix}$ has $k$ columns, and $\matT_k$ is a $k\times k$ upper triangular matrix.

After some algebra (see e.g. Section 3 of \cite{kimi2025kimi}), it can be shown that the state update reduces to few matrix products, notably, involving the inverse of the matrix $\matA$ defined as:
\begin{align*}
    \matA^{-1}:=\left(
        \matI + \diag(\vecb) 
        \cdot \left(\matGamma \odot \matK \right)
        \cdot \left(\matK \oslash \matGamma\right)^\top
        \odot \matL^{-}
    \right)^{-1},
\end{align*}
where $\matGamma$ is a matrix whose $i$-th row is equal to $\boldsymbol{\gamma}_i:=\veca_1 \odot \veca_2 \odot \ldots \odot \veca_i$.

\section{Deferred proofs}

\subsection{Proof of Proposition \ref{proposition:condition_number_upper_bound_delta_matrices}}
\label{appendix:proof_of_delta_matrix_inverse_elements}
\begin{proof}
        First note that $\matA$ has the following form:
        \begin{align*}
            \begin{pmatrix}
                1 & & & & \\
                \beta_1 \veck_1^\top \matD_{1}\matD_0^{-1} \veck_0 & 1 & & & \\
                \beta_2 \veck_2^\top \matD_2 \matD_0^{-1}\veck_0 & \beta_2 \veck_2^\top\matD_{2}\matD_1^{-1} \veck_1 & 1 & & \\
                % \beta_4 \veck_4^\top \veck_1 & \beta_4 \veck_4^\top \veck_2 & \beta_4 \veck_5^\top \veck_2 & 1& \\
                 \vdots& \vdots & & \ddots &  \\
                 \beta_{n-1} \veck_{n-1}^\top \matD_{n-1}\matD_0^{-1} \veck_0 & \beta_{n-1} \veck_{n-1}^\top \matD_{n-1}\matD_1^{-1}\veck_1 & \ldots & \ldots & 1
            \end{pmatrix},
        \end{align*}
        where $\matD_{i}$ is a diagonal matrix satisfying $\matD_{i}:=\diag(\vecgamma_i)=\diag(\veca_{0} \odot \veca_{1} \odot \ldots \odot \veca_{i})$.
        
        We know that $\matA\cdot \matA^{-1}=\matI$, which we can write as:
        \begin{align*}
            \begin{pmatrix}
                1 & & & & \\
                \matA_{1,0} & 1 & & & \\
                \matA_{2,0} & \matA_{2,1} & 1 & & \\
                % \beta_4 \veck_4^\top \veck_1 & \beta_4 \veck_4^\top \veck_2 & \beta_4 \veck_5^\top \veck_2 & 1& \\
                 \vdots& \vdots & & \ddots &  \\
                 \matA_{n-1,0} & \matA_{n-1,1} & \ldots & \matA_{n-1,n-2} & 1
            \end{pmatrix}
            \cdot 
            \begin{pmatrix}
                1 & & & & \\
                \vecx_{1,0} & 1 & & & \\
                \vecx_{2,0} & \vecx_{3,1} & 1 & & \\
                \vdots& \vdots & & \ddots &  \\
                \vecx_{n-1,0} & \vecx_{n-1,1} & \ldots & \vecx_{n-1,n-2} & 1
            \end{pmatrix} = 
            \matI,
        \end{align*}
        where $\vecx_{i,j}$ denotes the element of $\matA^{-1}$ at position $i,j$.

        Let us focus on bounding the magnitude of the elements of the first column of the inverse. By taking the inner product between the second row of $\matA$ and the first column of the inverse, we find that:
        \begin{align*}
            \vecx_{1,0} =
            -\matA_{1,0}
            = - \beta_1 \veck_1^\top \matD_{1}\matD_{0}^{-1} \veck_0,
        \end{align*}
        and since all the vectors $\veck_i$ have unit norm, it holds that $|\vecx_{1,0}| 
        \leq |\beta_1|\cdot \|\matD_{1}\matD_{0}^{-1}\|_2 \cdot \|\veck_1\|_2\cdot \|\veck_0\|_2 
        \leq \|\diag(\veca_{1})\|_2 
        \leq 1.$ A useful observation that we use in the rest of the proof is that, for all $i$ and $j<i$, it holds that $\|\matD_i\matD^{-1}_j\|_2=\|\|\diag(\veca_{j+1}\odot \veca_{j+2}\odot\ldots\odot \veca_{i})\|_2 \leq 1$.
        
        Now, for $i=1,\ldots,n-1$, $\vecx_{i,0}$ is given by:
        \begin{align*}
            \vecx_{i,0} = \beta_i\cdot \veck_i^\top \matD_i\left(
                \matD_0^{-1}\veck_0 \cdot \vecx_{0,0} + \matD_1^{-1}\veck_1 \cdot  \vecx_{1,0} + \ldots + \matD_{i-1}^{-1}\veck_{i-1} \cdot \vecx_{i-1,0}
            \right).
        \end{align*}
        
        For $\vecx_{2,0}$ we have that:
        \begin{align*}
            \vecx_{2,0} &= -\beta_2\cdot \veck_2^\top \matD_2\left(
                \matD_0^{-1}\veck_0 \cdot \vecx_{0,0} + \matD_1^{-1}\veck_1 \cdot  \vecx_{1,0}
            \right)
            \\
            &= -\beta_2\cdot \veck_2^\top\matD_2 \left(
                \matD_0^{-1}\veck_0 + \matD_1^{-1}\veck_1 \cdot  (- \beta_1 \veck_1^\top \matD_{1}\matD_{0}^{-1} \veck_0)
            \right)
            \\
            &= -\beta_2\cdot \veck_2^\top \matD_2\left(
                \matI - \beta_1 \cdot \matD_1^{-1}\veck_1  \veck_1^\top \matD_{1}
            \right)\matD_{0}^{-1} \veck_0
            \\
            &= -\beta_2\cdot \veck_2^\top \matD_2\left(
                \matD_1^{-1}\matD_1 - \beta_1 \cdot \veck_1  \veck_1^\top \matD_{1}
            \right)\matD_{0}^{-1} \veck_0
            \\
            &= -\beta_2\cdot \veck_2^\top \matD_2\matD_1^{-1}\left(
                \matI - \beta_1 \cdot \veck_1  \veck_1^\top 
            \right)\matD_1 \matD_{0}^{-1} \veck_0.
        \end{align*}
        Similarly, for $\vecx_{3,0}$:
        \begin{align*}
            \vecx_{3,0} &= -\beta_3\cdot \veck_3^\top \matD_3\left(
                \matD_0^{-1}\veck_0 \cdot \vecx_{0,0} + \matD_1^{-1}\veck_1 \cdot  \vecx_{1,0}
                +
                \matD_2^{-1}\veck_2 \cdot \vecx_{2,0}
            \right)
            \\
            &= -\beta_3\cdot \veck_3^\top \matD_3\Big[
                \matD_0^{-1}\veck_0  + \matD_1^{-1}\veck_1 \cdot  (- \beta_1 \veck_1^\top \matD_{1}\matD_{0}^{-1} \veck_0)
                \\
                &\qquad+
                \matD_2^{-1}\veck_2 \cdot 
                \left(-\beta_2\cdot \veck_2^\top \matD_2\matD_1^{-1}\left(
                \matI - \beta_1 \cdot \veck_1  \veck_1^\top
            \right)\matD_1 \matD_{0}^{-1} \veck_0
            \right)
            \Big]
            \\
            &= -\beta_3\cdot \veck_3^\top \matD_3\Big[
                \matI  - \beta_1 \cdot \matD_1^{-1}\veck_1 \veck_1^\top \matD_{1}
                \\
                &\qquad-\beta_2\cdot
                \matD_2^{-1}\veck_2
                \veck_2^\top \matD_2\matD_1^{-1}\left(
                \matI - \beta_1 \cdot \veck_1  \veck_1^\top 
            \right)\matD_1
            \Big] \matD_{0}^{-1} \veck_0
            \\
            &= -\beta_3\cdot \veck_3^\top \matD_3\left(
                \matI -\beta_2\cdot
                \matD_2^{-1}\veck_2
                \veck_2^\top \matD_2
            \right) \matD_1^{-1}\left(
                \matI - \beta_1 \cdot \veck_1  \veck_1^\top 
            \right)\matD_1\matD_{0}^{-1} \veck_0
            \\
            &= -\beta_3\cdot \veck_3^\top \matD_3\matD_2^{-1}\left(
                \matI -\beta_2\cdot
                \veck_2
                \veck_2^\top 
            \right)\matD_2 \matD_1^{-1}\left(
                \matI - \beta_1 \cdot \veck_1  \veck_1^\top 
            \right)\matD_1\matD_{0}^{-1} \veck_0
            .
        \end{align*}
        By applying Cauchy-Schwarz in the last equation, we conclude that $|\vecx_{3,0}|\leq 1$.
        
        Now, we will prove by induction that if $\vecx_{l,0}$ can be written as:
        \begin{align*}
            -\beta_l\veck_l^\top \matD_{l}\matD_{l-1}^{-1}\left(
                \matI - \beta_{l-1}\veck_{l-1}\veck_{l-1}^\top
            \right)
            \matD_{l-1}\matD_{l-2}^{-1}
            \left(
                \matI - \beta_{l-2}\veck_{l-2}\veck_{l-2}^\top
            \right)
            \ldots
            \left(
                \matI - \beta_1\veck_1\veck_1^\top
            \right)
            \matD_1\matD_0^{-1}
            \veck_0,
        \end{align*}
        for all $l=1,2,\ldots,i-1$, then the same holds for $\vecx_{i}$.
        The claim holds for the cases $l=1,2,3$, as described above. Now, we write:
        \begingroup
        \allowdisplaybreaks
        \begin{align*}
            \vecx_{i,0} 
            &= - \beta_{i}\cdot \veck_{i}^\top \matD_i
            \left(\sum_{l=0}^{i-1} \matD_l^{-1}\veck_l\cdot \vecx_{l,0}\right)\\
            &= - \beta_{i}\cdot \veck_{i}^\top  \matD_i\left(
                \matD_0^{-1}\veck_0 + \matD_1^{-1}\veck_1\cdot \vecx_{1,0}
                + \matD_2^{-1}\veck_2\cdot \vecx_{2,0}
                + \ldots
                + \matD_i^{-1}\veck_{i-1}\cdot \vecx_{i-1,0}
            \right)
            \\
            &= - \beta_{i}\cdot \veck_{i}^\top \matD_i\Big[
                \matD_0^{-1}\veck_0 
                + \matD_1^{-1}\veck_1\cdot (- \beta_1 \veck_1^\top \matD_{1}\matD_{0}^{-1} \veck_0)
                \\
                &\quad\ldots + \matD_2^{-1}\veck_2\cdot 
                \left(
                    -\beta_2\cdot \veck_2^\top \matD_2\matD_1^{-1}\left(
                    \matI - \beta_1 \cdot \veck_1  \veck_1^\top 
                \right)\matD_1 \matD_{0}^{-1} \veck_0
                \right) 
                \\
                &\quad\ldots+ \matD_3^{-1}\veck_3\cdot 
                \left(
                    -\beta_3\cdot \veck_3^\top \matD_3\matD_2^{-1}\left(
                        \matI -\beta_2\cdot
                        \veck_2
                        \veck_2^\top 
                    \right)\matD_2 \matD_1^{-1}\left(
                        \matI - \beta_1 \cdot \veck_1  \veck_1^\top 
                    \right)\matD_1\matD_{0}^{-1} \veck_0
                \right)
                \\
                &\quad\ldots+\ldots
                \\
                &\quad\ldots+ \matD_{i-1}^{-1}\veck_{i-1}\cdot 
                \left(
                    -\beta_{i-1}\veck_{i-1}^\top \matD_{i-1}\matD_{i-2}^{-1}\left(
                        \matI - \beta_{i-2}\veck_{i-2}\veck_{i-2}^\top
                    \right)
                    \ldots
                    \left(
                        \matI - \beta_1\veck_1\veck_1^\top
                    \right)
                    \matD_1\matD_0^{-1}
                    \veck_0
                \right)
                \Big]
            \\
            &= - \beta_{i}\cdot \veck_{i}^\top \matD_i\Big[
                \matI 
                - \beta_1 \matD_1^{-1}\veck_1 \veck_1^\top \matD_{1}
                \\
                &\quad\ldots -\beta_2 \matD_2^{-1}\veck_2 \veck_2^\top \matD_2\matD_1^{-1}\left(
                    \matI - \beta_1 \cdot \veck_1  \veck_1^\top 
                \right)\matD_1 
                \\
                &\quad\ldots -\beta_3 \matD_3^{-1}\veck_3 \veck_3^\top \matD_3\matD_2^{-1}\left(
                        \matI -\beta_2\cdot
                        \veck_2
                        \veck_2^\top 
                    \right)\matD_2 \matD_1^{-1}\left(
                        \matI - \beta_1 \cdot \veck_1  \veck_1^\top 
                    \right)\matD_1
                \\
                &\quad\ldots-\ldots
                \\
                &\quad\ldots-\beta_{i-1} \matD_{i-1}^{-1}\veck_{i-1}
                    \veck_{i-1}^\top \matD_{i-1}\matD_{i-2}^{-1}\left(
                        \matI - \beta_{i-2}\veck_{i-2}\veck_{i-2}^\top
                    \right)
                    \ldots
                    \left(
                        \matI - \beta_1\veck_1\veck_1^\top
                    \right)
                    \matD_1
                \Big]\matD_0^{-1}\veck_0
            \\
            &= - \beta_{i}\cdot \veck_{i}^\top \matD_i\Big[
                \matI -\beta_2 \matD_2^{-1}\veck_2 \veck_2^\top \matD_2 
                \\
                &\quad\ldots -\beta_3 \matD_3^{-1}\veck_3 \veck_3^\top \matD_3\matD_2^{-1}\left(
                        \matI -\beta_2\cdot
                        \veck_2
                        \veck_2^\top 
                    \right)\matD_2
                \\
                &\quad\ldots-\ldots
                \\
                &\quad\ldots-\beta_{i-1} \matD_{i-1}^{-1}\veck_{i-1}
                    \veck_{i-1}^\top \matD_{i-1}\matD_{i-2}^{-1}\left(
                        \matI - \beta_{i-2}\veck_{i-2}\veck_{i-2}^\top
                    \right)
                    \ldots
                    \left(
                        \matI - \beta_2\veck_2\veck_2^\top
                    \right)
                    \matD_2
                \\
                &\quad\Big]
                \matD_1^{-1}(
                \matI 
                - \beta_1 \veck_1 \veck_1^\top) \matD_{1}\matD_0^{-1}\veck_0
            \\
            &= - \beta_{i}\cdot \veck_{i}^\top \matD_i\Big[
                \matI -\beta_3 \matD_3^{-1}\veck_3 \veck_3^\top \matD_3 
                \\
                &\quad\ldots -\beta_4 \matD_4^{-1}\veck_4 \veck_4^\top \matD_4\matD_3^{-1}\left(
                        \matI -\beta_3\cdot
                        \veck_3
                        \veck_3^\top 
                    \right)\matD_3
                \\
                &\quad\ldots-\ldots
                \\
                &\quad\ldots-\beta_{i-1} \matD_{i-1}^{-1}\veck_{i-1}
                    \veck_{i-1}^\top \matD_{i-1}\matD_{i-2}^{-1}\left(
                        \matI - \beta_{i-2}\veck_{i-2}\veck_{i-2}^\top
                    \right)
                    \ldots
                    \left(
                        \matI - \beta_3\veck_3\veck_3^\top
                    \right)
                    \matD_3
                \\
                &\quad\Big]
                \matD_2^{-1}\left(
                        \matI -\beta_2
                        \veck_2
                        \veck_2^\top 
                    \right)\matD_2
                \matD_1^{-1}(
                \matI 
                - \beta_1 \veck_1 \veck_1^\top) \matD_{1}\matD_0^{-1}\veck_0
            \\
            &=\ldots
            \\
            &= -\beta_i\veck_i^\top \matD_{i}\matD_{i-1}^{-1}\left(
                \matI - \beta_{i-1}\veck_{i-1}\veck_{i-1}^\top
            \right)
            \matD_{i-1}\matD_{i-2}^{-1}
            \ldots
            \left(
                \matI - \beta_1\veck_1\veck_1^\top
            \right)
            \matD_1\matD_0^{-1}
            \veck_0,
        \end{align*}
        \endgroup
        as claimed. This allows us to bound $|\vecx_{i,0}|\leq 1$ by applying Cauchy-Schwarz, for every $i=1,\ldots, n-1$. The exact same analysis can be carried out for all the other columns of the inverse, besides the first one.
    \end{proof}

\subsection{Proof of MBH complexity}
\label{appendix:proof_of_lemma_unrolled_mbh_derivation}
\begin{lemma}
    \label{lemma:unrolled_mbh_derivation}
    Let $\matA=\matI+\matL$ be an $n\times n$ matrix where $\matL$ is strictly lower triangular. We can compute the inverse of $\matA$ using $2\log(n)$ matrix products of size $n\times n$, by unrolling the  algorithm of \cite{bunch1974triangular}.
    \begin{proof}
        Assume for simplicity that $n$ is a power of $2$, and let $b\in \{2, 4, 8, \ldots, n/2\}$. We define the matrix
        \begin{equation}
        \matD(\matA,b) =
        \begin{pmatrix}
        \matA_{0,0} & 0 & \cdots & 0 & 0 \\
        0 & \matA_{1,1} & \cdots & 0 & 0 \\
        \vdots &  & \ddots & & \vdots \\
        0 & 0 & \cdots & 0 & \matA_{n/b-1,n/b-1}
        \end{pmatrix}.
        \end{equation}
        
        Here, $\matD(\matA,b)$ contains only the diagonal blocks of the matrix $\matA$. Each block $\matA_{i,i}$ has size $b \times b$, and there are $n/b$ blocks in total. Let also $\matD_{e}(\matA,b)$ be the matrix containing the ``even'' diagonal blocks of $\matD(\matA,b)$, and  $\matD_{o}(\matA,b)$ the odd blocks of $\matD(\matA,b)$. Clearly, $\matD_{o}(\matA,b) + \matD_{e}(\matA,b) = \matD(\matA,b)$. Now, we work as follows. First, we initialize $\matX = \matI$ and $b=1.$ Then, while $b<n$, we execute the following three steps:
        \begin{enumerate}
        \item Create $\matD_{e}(\matX,b)$ and $\matD_{o}(\matX,b)$ as defined above.
        \item Update $
        \matX \leftarrow \matD_{e}(\matX,b) + \matD_{o}(\matX,b) - \matD_{o}(\matX,b) \cdot \matL \cdot \matD_{e}(\matX,b).
        $
        \item $b\leftarrow 2b.$
        \end{enumerate}
        Finally, we return $\matX.$ In total there are two matrix products of size $n\times n$ for each $b=1,2,4,\ldots,n/2$, which gives a total of $2\log(n)$ matrix products.

        Correctness can be verified by induction. Note that, for $b=1$, after step 2., the $2\times 2$ diagonal blocks of $\matX$ contain the inverses of the $2\times 2$ diagonal blocks of $\matA$. Since $\matX=\matI$ when $b=1$, we can write:
        \begin{align*}
            \matD_e(\matX,1) = \begin{pmatrix}
                1 & 0 & 0 & 0 & 0 &\ldots & 0 & 0 \\
                0 & 0 & 0 & 0 & 0 &\ldots & 0 & 0 \\
                0 & 0 & 1 & 0 & 0 &\ldots & 0 & 0 \\
                0 & 0 & 0 & 0 & 0 &\ldots & 0 & 0 \\
                0 & 0 & 0 & 0 &1 &\ldots & 0 & 0 \\
                 &  & \vdots &  & & \ddots &  &  \\
                0 & 0 & 0 & 0 & 0& \ldots & 1 & 0 \\
                0 & 0 & 0 & 0 & 0& \ldots & 0 & 0 \\
            \end{pmatrix},
            \quad
            \matD_o(\matX,1) = \begin{pmatrix}
                0 & 0 & 0 & 0 & 0 &\ldots & 0 & 0 \\
                0 & 1 & 0 & 0 & 0 &\ldots & 0 & 0 \\
                0 & 0 & 0 & 0 & 0 &\ldots & 0 & 0 \\
                0 & 0 & 0 & 1 & 0 &\ldots & 0 & 0 \\
                0 & 0 & 0 & 0 & 0 &\ldots & 0 & 0 \\
                 &  & \vdots &  & & \ddots &  &  \\
                0 & 0 & 0 & 0 & 0& \ldots & 0 & 0 \\
                0 & 0 & 0 & 0 & 0& \ldots & 0 & 1 \\
            \end{pmatrix},
        \end{align*}
        and it is easy to verify that the elements of $\matI-\matD_o(\matX,1) \cdot \matL \cdot \matD_e(\matX,1)$ are equal to:
        \begin{align*}
            \left(\matI-\matD_o(\matX,1) \cdot \matL \cdot \matD_e(\matX,1)\right)_{i,j}
            =
            \begin{cases}
                1, &i=j,\\
                -\matL_{i,j}, & j<i,  i=2k, j=2k-1, k=\{1,2,\ldots,n/2\},\\
                0, &\text{otherwise}.
            \end{cases}
        \end{align*}
        This gives us the base case for the induction. We next argue that, if $\matX$ after step 2. contains the inverses of the diagonal blocks of $\matA:=\matI+\matL$ of size $2b\times 2b$ for any $b=1,2,4,\ldots,n/4$, then the same holds for $2b$. By the induction hypothesis, at step 1., $\matD_e(\matX,2b)$ contains the inverses of the even diagonal $2b\times 2b$ blocks of $\matA$, and $\matD_o(\matX,2b)$ the odd blocks. Therefore they can be written as:
        \begin{align*}
            \matD_e(\matX,2b) &= \begin{pmatrix}
                \matA_{0,0}^{-1} &  &  &  &  & &  &  \\
                 & 0_{2b\times 2b} &  &  &  & &  &  \\
                 &  & \matA_{2,2}^{-1} &  & & &  &  \\
                 &  &  & 0_{2b\times 2b} &  & &  &  \\
                 &  &  &  &\matA_{4,4}^{-1} & &  &  \\
                 &  &  &  & & \ddots &  &  \\
                 &  &  &  & & & \matA_{\frac{n}{2b}-2,
                \frac{n}{2b}-2}^{-1} &  \\
                 &  &  &  & &  &  & 0_{2b\times 2b} \\
            \end{pmatrix},
            \end{align*}
            and
            \begin{align*}\\
            \matD_o(\matX,2b) &= \begin{pmatrix}
                0_{2b\times 2b} &  &  &  &  & &  &  \\
                &\matA_{1,1}^{-1} &  &  &  &  & &   \\
                & & 0_{2b\times 2b} &  &  &  & &   \\
                & &  & \matA_{3,3}^{-1} &  & & &   \\
                & &  &  & 0_{2b\times 2b} &  & &   \\
                & &  &  &  & \ddots&  &   \\
                & &  &  &  & & 0_{2b\times 2b} &  \\
                & &  &  &  & & & \matA_{\frac{n}{2b}-1,
                \frac{n}{2b}-1}^{-1}  \\
            \end{pmatrix}
        \end{align*}
         Then, the matrix $\matD_{e}(\matX,2b) + \matD_{o}(\matX,2b)-\matD_o(\matX,2b) \cdot \matL \cdot \matD_e(\matX,2b)$ has the form:
         \begin{align*}
             \begin{pmatrix}
                    \matA_{0,0}^{-1} 
                    &  
                    &  
                    &  
                    & 
                    & 
                    &  
                \\
                    -\matA_{1,1}^{-1}\matL_{1,0}\matA_{0,0}^{-1}
                    &\matA_{1,1}^{-1} 
                    & 
                    & 
                    & 
                    & 
                    & 
                \\
                    0_{2b\times 2b}
                    & 0_{2b\times 2b}
                    & \matA_{2,2}^{-1} 
                    & 
                    & 
                    &  
                    & 
                \\
                    -\matA_{3,3}^{-1}\matL_{3,0}\matA_{0,0}^{-1}
                    & 0_{2b\times 2b}
                    & -\matA_{3,3}^{-1}\matL_{3,2}\matA_{2,2}^{-1}  
                    & \matA_{3,3}^{-1} 
                    & 
                    &  
                    & 
                \\
                    0_{2b\times 2b}
                    & 0_{2b\times 2b} 
                    &  0_{2b\times 2b} 
                    &  0_{2b\times 2b} 
                    & \matA_{4,4}^{-1} 
                    &  
                    &
                \\
                -\matA_{5,5}^{-1}\matL_{5,0}\matA_{0,0}^{-1} 
                    & 0_{2b\times 2b} 
                    &  -\matA_{5,5}^{-1}\matL_{5,2}\matA_{2,2}^{-1}  
                    &  0_{2b\times 2b} 
                    &  -\matA_{5,5}^{-1}\matL_{5,4}\matA_{4,4}^{-1}
                    & -\matA_{5,5}^{-1}
                    &
                \\
                    \vdots
                    & \vdots
                    & 
                    & 
                    & 
                    & 
                    & 
                    \ddots  
                \\
            \end{pmatrix}.
         \end{align*}
         Based on Eq. \eqref{eq:recursive_bunch_hopcroft}, each $4b\times 4b$ diagonal block of the previous matrix contains the inverse of the corresponding $4b\times 4b$ diagonal block of $\matA$. This completes the induction step and concludes the proof. 
    \end{proof}
\end{lemma}

\subsection{Stability of MCH}
\label{appendix:proof_of_stability_of_mch}
    \begin{proof}
        
        Let us first bound the errors for the repeated squaring part, which is completely independent from the computation of the inverse. That is, we focus on lines 3 and 4 of Algorithm \ref{alg:mch}. Before we start the analysis, we state some preliminaries and define notation. Using standard bounds on floating point analysis of matrix multiplication, we assume access to an algorithm that computes a matrix $\matCtilde$ that approximates the product $\matA$,$\matB$ such that:
        \begin{align}
            \|\matCtilde-\matA\matB\|_2\leq \mu(n)\cdot \|\matA\|_2\cdot \|\matB\|_2 \cdot \umach,
            \label{eq:matmul_floating_point_error_bound}
        \end{align}
        where $\mu(n)$ is some low-degree polynomial in $n$.
        
        We also make the following explicit assumption on $\umach$ which is necessary for our analysis to yield the advertised bounds:
        \begin{align}
            \umach \leq \frac{1}{2^n\cdot \mu(n)}.
            \label{eq:assumed_bound_on_umach_mch}
        \end{align}
        
        Now, let us define $\matYtilde_k$ to be the matrix that is being computed numerically at iteration $k=1,2,\ldots,\log(n/2)$, and let $\matY_k=\matY^{2^k}$ be the true matrix at iteration $k$ if the computation was carried out in infinite precision. 
        
        From Eq. \eqref{eq:matmul_floating_point_error_bound}, $\matYtilde_k$ satisfies:
        \begin{align*}
            \matYtilde_k = \matYtilde_{k-1}^2 + \matE^{MM}_k,
        \end{align*}
        where $\matE^{MM}_k$ is an error matrix that arises due to floating point errors, and it satisfies $\|\matE^{MM}_k\|_2\leq \mu(n)\cdot \|\matYtilde_{k-1}^2\|_2 \cdot \umach$.
        \begin{align*}
            \|\matYtilde_k\|_2 &= \|\matYtilde_{k-1}^2 + \matE^{MM}_k\|_2 \\
                &\leq \|\matYtilde_{k-1}\|_2^2 + \|\matE^{MM}_k\|_2 \\
                &\leq \|\matYtilde_{k-1}\|_2^2 + \mu(n)\cdot \|\matYtilde_{k-1}^2\|_2 \cdot \umach \\
                &= \|\matYtilde_{k-1}\|_2^2 \cdot (1+\mu(n)\cdot\umach) \\
                &\leq 2\cdot \|\matYtilde_{k-1}\|_2^2,
        \end{align*}
        where in the last step we used Inequality \eqref{eq:assumed_bound_on_umach_mch}. 
        Recall that for $k=1$, we have that $\|\matYtilde_1\|_2 \leq \|\matY_0\|_2^2$, where $\matY_0$ is the initial matrix. Using this as a base case and solving the recursion, it follows that:
        \begin{align}
            \|\matYtilde_{k}\|_2 \leq 2\cdot 2^2 \cdot 2^4\cdot \ldots \cdot 2^{2^{k-1}} \cdot \|\matYtilde_1\|_2^{2^k}
            \leq 2^{2^{k}-1} \cdot \|\matY_0\|_2^{2^{k+1}}.
            \label{eq:bound_norm_y_tilde_k}
        \end{align}
        This allows us to directly obtain a bound on the matrix multiplication error at iteration $k$:
        \begin{align*}
           \|\matE^{MM}_k\|_2 &\leq \mu(n)\cdot \|\matYtilde^2_{k-1}\|_2\cdot \umach \\
            &\leq \mu(n)\cdot 2^{2^{k-1}-1} \cdot \|\matY_0\|_2^{2^{k}} \cdot \umach.
        \end{align*}
        Now, let $\matE_k:=\matYtilde_k-\matY_0^{2^k}$ be the total error matrix at iteration $k$. We can write:
        \begin{align*}
            \matY_0^{2^k} + \matE_k &= \matYtilde_{k-1}^2 + \matE^{MM}_k \\
                &= (\matY_0^{2^{k-1}} + \matE_{k-1})^2 + \matE^{MM}_k \\
                &= \matY_0^{2^{k}} + \matY_0^{2^{k-1}}\matE_{k-1} + \matE_{k-1} \matY_0^{2^{k-1}} + \matE_{k-1}^2 + \matE^{MM}_k.
        \end{align*}
        In the equation above, the terms $\matY_0^{2^k}$ on the left and the right hand side cancel out. If we take norms on both sides, we can write that:
        \begin{align}
            \|\matE_k\|_2 &\leq 2\cdot \|\matY_0^{2^{k-1}}\|_2\cdot \|\matE_{k-1}\|_2 + \|\matE_{k-1}\|_2^2 + \|\matE^{MM}_k\|_2 \\
                &\leq 2\cdot \|\matY_0\|_2^{2^{k-1}}\cdot \|\matE_{k-1}\|_2 + \|\matE_{k-1}\|_2^2 + \mu(n)\cdot 2^{2^{k-1}-1} \cdot \|\matY_0\|_2^{2^{k}} \cdot \umach.
            \label{eq:first_bound_on_err_k}
        \end{align}
        We will now prove that $\|\matE_k\|_2\leq 2^{2^{k}} \cdot \|\matY_0\|^{2^{k}} \cdot \mu(n)\cdot \umach$ by induction. 
        For the inductive step, we assume that $\|\matE_{k}\|_2\leq 2^{2^k}\cdot \|\matY_0\|^{2^k}\cdot \mu(n)\cdot \umach$. It trivially holds in the base case $k=1$ from Inequality \eqref{eq:matmul_floating_point_error_bound}. We now prove that if it holds for $k$, then it holds also for $k+1$. Inequality \eqref{eq:first_bound_on_err_k} implies that:
        \begin{align*}
            \|\matE_{k+1}&\|_2 \leq 2\|\matY_0\|_2^{2^k} \cdot \|\matE_k\|_2 + \|\matE_k\|_2^2 + 2^{2^k-1} \cdot \|\matY_0\|_2^{2^{k+1}} \cdot \mu(n) \cdot \umach  \\
            &\leq 2\|\matY_0\|_2^{2^k} \cdot 2^{2^k} \|\matY_0\|_2^{2^k} \mu(n)\cdot \umach + (2^{2^k} \|\matY_0\|_2^{2^k} \mu(n)\cdot \umach)^2 + \mu(n) \cdot 2^{2^k-1} \|Y_0\|_2^{2^{k+1}} \cdot \umach \\
            &= \|\matY_0\|_2^{2^{k+1}}\cdot\mu(n)\cdot \umach \cdot (2\cdot2^{2^{k}}+2^{2^{k+1}}\cdot\mu(n)\cdot \umach+2^{2^k-1}) \\
            &\leq \|\matY_0\|_2^{2^{k+1}}\cdot\mu(n)\cdot \umach \cdot (2\cdot2^{2^{k}}+2^{2^{k}}+2^{2^k-1}) \qquad\text{(due to Ineq. \ref{eq:assumed_bound_on_umach_mch})}\\
            &= \frac{7}{2}\cdot2^{2^k} \cdot\|\matY_0\|_2^{2^{k+1}}\cdot\mu(n)\cdot \umach \\
            &\leq 2^{2^{k+1}}\cdot \|\matY_0\|_2^{2^{k+1}}\cdot\mu(n)\cdot \umach,
        \end{align*}
        where the last holds for all $k\geq 2$.
        
        Next, we shift the attention to the computation of the final $\matX$ which approximates the inverse. Let $\matXtilde_k$ be the matrix computed at iteration $k\in\{1,\ldots,\log(n/2)\}$ in Algorithm \ref{alg:mch}, line 5, given by the formula $\matXtilde_k\gets \matXtilde_{k-1}+\matXtilde_{k-1}\cdot \matYtilde_k$, where $\matXtilde_0:=\matI-\matL$.
        
        We recall that for the element-wise floating point addition between two matrices $\matA$ and $\matB$ the returned matrix $\matCtilde=\fl(\matA+\matB)$ satisfies:
        \begin{align}
            \|\matCtilde - (\matA+\matB)\|_2 \leq \umach\cdot \sqrt{n}\cdot \|\matA+\matB\|_2.
            \label{eq:floating_point_addition_error_bound}
        \end{align}
        
        Using this we can calculate a bound on the error of $\matXtilde_k$:
        \begin{align*}
            \matXtilde_k = \matXtilde_{k-1} + \matXtilde_{k-1}\cdot \matYtilde_k + \matF^{GEMM}_k,
        \end{align*}
        where $\matF^{GEMM}_k$ denotes the total error of the ``GEMM'' operation $\matXtilde_{k-1} + \matXtilde_{k-1}\cdot \matYtilde_k$. Let $\matZtilde_k:=\fl(\matXtilde_{k-1}\cdot\matYtilde_k)=\matXtilde_{k-1}\cdot\matYtilde_k + \matF^{MM}_k$, where $\matF^{MM}_k$  is the error of the multiplication $\matXtilde_{k-1}\cdot \matYtilde_k$. It satisfies $\|\matF^{MM}_k\| \leq \mu(n)\cdot \umach \cdot \|\matXtilde_{k-1}\|_2\cdot \|\matYtilde_k\|_2$, which implies that $\|\matZtilde_k\|_2 \leq \|\matXtilde_{k-1}\|_2\cdot \|\matYtilde_k\|_2\cdot (1+\mu(n)\cdot \umach) $. For the subsequent addition, we have that:
        \begin{align*}
            \fl(\matXtilde_{k-1} + \matZtilde_k) &:= \matXtilde_{k-1} + \matZtilde_k + \matF^{ADD}_k \\
                &= \matXtilde_{k-1} +\matXtilde_{k-1}\cdot\matYtilde_k + \matF^{MM}_k + \matF^{ADD}_k,
        \end{align*}
        where, from Inequality \eqref{eq:floating_point_addition_error_bound}, it holds that:\begin{align*}
        \|\matF^{ADD}_k\|_2 &\leq \umach \cdot \sqrt{n} \cdot \|\matXtilde_{k-1} + \matZtilde_k\|_2\\
            &\leq \umach \cdot \sqrt{n} \cdot \left(
                \|\matXtilde_{k-1}\|_2 + \|\matXtilde_{k-1}\|_2\cdot \|\matYtilde_k\|_2\cdot (1+\mu(n)\cdot \umach)
            \right) \\
            &= \umach \cdot \sqrt{n} \cdot \|\matXtilde_{k-1}\|_2 \cdot 
            \left(
                1 + \|\matYtilde_k\|_2\cdot (1+\mu(n)\cdot \umach)
            \right).
        \end{align*}
        This finally gives:
        \begin{align*}
            \|\matF^{GEMM}_k&\|_2 = \|\matF^{ADD}_k + \matF^{MM}_k\|_2 \\
                &\leq 
                \mu(n)\cdot \umach \cdot \|\matXtilde_{k-1}\|_2\cdot \|\matYtilde_k\|_2
                +
                \umach \cdot \sqrt{n} \cdot \|\matXtilde_{k-1}\|_2 \cdot 
                \left(
                    1 + \|\matYtilde_k\|_2\cdot (1+\mu(n)\cdot \umach)
                \right)\\
                &\leq 
                \mu'(n)\cdot \umach \cdot \|\matXtilde_{k-1}\|_2
                \cdot 
                \left(
                        1 + \|\matYtilde_k\|_2\cdot (2+\mu(n)\cdot \umach)
                \right) \\
                &\leq 
                \mu'(n)\cdot \umach \cdot \|\matXtilde_{k-1}\|_2
                \cdot 
                \left(
                        1 + 2^{2^{k}-1} \cdot \|\matY_0\|_2^{2^{k+1}}\cdot (2+\mu(n)\cdot \umach)
                \right)
                \\
                &\leq 
                \mu'(n)\cdot \umach \cdot \|\matXtilde_{k-1}\|_2
                \cdot 
                \left(
                        1 + 2^{2^{k}+1} \cdot \|\matY_0\|_2^{2^{k+1}}
                \right),
        \end{align*}
        where $\mu'(n)$ (similar to $\mu(n)$) is polynomial in $n$. We can get a bound on $\|\matXtilde_{k}\|_2$ in a similar way that we did with $\|\matYtilde_k\|_2$. We have that:
        \begin{align}
            \|\matXtilde_{k}\|_2 &= \|\matXtilde_{k-1}+\matXtilde_{k-1}\matYtilde_{k} + \matF^{GEMM}_k\|_2
                \nonumber\\
                &\leq \|\matXtilde_{k-1}\|_2 \cdot \left(
                    1+\|\matYtilde_k\|_2 + \mu'(n)\cdot \umach 
                    \cdot 
                    \left(
                            1 + 2^{2^{k}+1} \cdot \|\matY_0\|_2^{2^{k+1}}
                    \right)
                \right)
                \nonumber\\
                &\leq \|\matXtilde_{k-1}\|_2 \cdot 
                    \left(
                        1 + 2^{2^{k}+1} \cdot \|\matY_0\|_2^{2^{k+1}}
                    \right)
                    \cdot
                    \left(
                        1+\mu'(n)\cdot \umach 
                    \right)
                \nonumber\\
                &\leq 2\cdot\|\matXtilde_{k-1}\|_2 \cdot 
                    \left(
                        1 + 2^{2^{k}+1} \cdot \|\matY_0\|_2^{2^{k+1}}
                    \right)
                \nonumber\\
                &\leq \|\matX_0\|_2 
                    \cdot
                    2^{2^k}
                    \cdot
                    \prod_{i=1}^k
                        \left(
                            1 + 2^{2^{i}+1} \cdot \|\matY_0\|_2^{2^{i+1}}
                        \right)
                \nonumber\\
                &\leq \|\matX_0\|_2 
                    \cdot
                    2^{2^k}
                    \cdot
                    \prod_{i=1}^k
                        \left(
                            1 + (2\cdot \|\matY_0\|_2)^{2^{i+1}}
                        \right)\\
                &\leq (1+\|\matY_0\|_2) 
                    \cdot
                    2^{2^k}
                    \cdot
                    \prod_{i=1}^k
                        \left(
                            1 + (2\cdot \|\matY_0\|_2)^{2^{i+1}}
                        \right).
                \label{eq:bound_on_x_tilde_k}
        \end{align}
        
        Finally, we can consider the final error of the algorithm. Let us write $
            \matXtilde_k = \matX_k + \matF_k,$
        where now $\matF_k$ contains all the errors of the final returned matrix. By using the same methodology as before with $\matY_k$ and $\matE_k$, for all $k=1,\ldots,\log(n/2)$ we can write:
        \begin{align*}
            \matX_k + \matF_k &= (\matX_{k-1}+\matF_{k-1}) + (\matX_{k-1}+\matF_{k-1})(\matY_k+\matE_k) + \matF^{GEMM}_k
            \\
                &=\matX_{k-1}+\matF_{k-1} + \matX_{k-1}\matY_{k} + \matX_{k-1}\matE_k + \matF_{k-1}\matY_k + \matF_k\matE_k + \matF^{GEMM}_k,
        \end{align*}
        where $\matF_0=0$. This gives:
        \begin{align}
            \matF_k &= \matF_{k-1} + \matX_{k-1}\matE_k + \matF_{k-1}\matY_k + \matF_k\matE_k + \matF^{GEMM}_k \nonumber
                \\
                &= \matF_{k-1}\left(\matI + \matY_k + \matE_k\right) + \matX_{k-1}\matE_k + \matF^{GEMM}_k \nonumber
                \\
                &= \matF_{k-1}\left(\matI + \matY_k + \matE_k\right) + \matX_{k-1}\matE_k + \matF^{GEMM}_k.
        \end{align}
        Let us denote $a(k):=\|\matI + \matY_k + \matE_k\|_2 $ and $\beta(k):=\|\matX_{k-1}\matE_k + \matF^{GEMM}_k\|_2$. By taking the norm of $\matF_k$ we arrive to a first-order linear recurrence:
        \begin{align}
            \|\matF_k\|_2 &\leq \|\matF_{k-1}\|_2\cdot a(k) + \beta(k) \leq \beta(k) + \sum_{j=1}^{k-1} \beta(j) \cdot\prod_{i=j+1}^{k}a(i)
            \label{eq:bound_on_matf_k_initial}
        \end{align}
        Let us now bound the terms $a(k)$ and $\beta(k)$. We have that:
        \begin{align*}
            a(k) &= \|\matI + \matY_k + \matE_k\|_2 \\
                &\leq \left(1 + \|\matY_0\|_2^{2^k} + 2^{2^k}\cdot \|\matY_0\|_2^{2^k}\cdot \mu(n)\cdot \umach\right)
                \\
                &\leq 1 + 2\|\matY_0\|_2^{2^k}.
        \end{align*}
        
        Let $\psi(k,c,d):=\prod_{i=1}^k (1+(c\cdot \|\matY_0\|_2)^{2^{i+d}})$.
        The products of $a(k)$'s can be bounded as:
        \begin{align*}
             \prod_{i=1}^k a(i) &\leq
                \prod_{i=1}^k (1 + 2\|\matY_0\|_2^{2^k})
                \leq
                \prod_{i=1}^k (1 + (2\cdot \|\matY_0\|_2)^{2^k})
                =
                \psi(k,2,0).
        \end{align*}
        For $\beta(k)$, we have that:
        \begin{align}
            \beta(k) &= \|\matX_{k-1}\matE_k + \matF^{GEMM}_k\|_2 \nonumber\\
                &\leq \|\matX_{k-1}\|_2\cdot \|\matE_k\|_2 + \|\matF^{GEMM}_k\|_2 \nonumber\\
                &\leq \|\matX_{k-1}\|_2\cdot 2^{2^k}\cdot \|\matY_0\|_2^{2^k}\cdot \mu(n)\cdot \umach  +
                \mu'(n)\cdot \umach \cdot \|\matXtilde_{k-1}\|_2
                \cdot 
                \left(
                        1 + 2^{2^{k}+1} \cdot \|\matY_0\|_2^{2^{k+1}}
                \right)
                \nonumber\\
                &\leq 
                \mu''(n)\cdot \umach
                \cdot 
                \left[
                    \|\matX_{k-1}\|_2\cdot 2^{2^k}\cdot \|\matY_0\|^{2^k} 
                    +
                 \|\matXtilde_{k-1}\|_2
                \cdot 
                \left(
                        1 + 2^{2^{k}+1} \cdot \|\matY_0\|_2^{2^{k+1}}
                \right)
                \right],
                \label{eq:initial_bound_on_beta_k}
        \end{align}
        where $\mu''(n)=\max(\mu(n),\mu'(n))$. To proceed further, we bound:
        \begin{align}
            \|\matX_{k-1}\|_2 &= \|\matX_{k-1}+\matX_{k-1}\matY_{k}\|_2 \nonumber\\
                &\leq \|\matX_{k-1}\|\cdot \|\matI+\matY_{k}\|_2 \nonumber\\
                &\leq \|\matX_{k-1}\|\cdot (1+\|\matY_0\|_2^{2^k}) \nonumber\\
                &\leq \prod_{i=1}^k (1+\|\matY_0\|_2^{2^i}) \nonumber \\
                &= \psi(k,1,0).
                \label{eq:bound_on_x_k}
        \end{align}
        Combining Inequalities \eqref{eq:bound_on_x_tilde_k}, \eqref{eq:initial_bound_on_beta_k}, and \eqref{eq:bound_on_x_k}, we obtain:
        \begin{align}
            \beta&(k) \leq 
                    \mu''(n)\cdot \umach
                    \cdot 
                    \left[
                        \|\matX_{k-1}\|_2\cdot 2^{2^k}\cdot \|\matY_0\|_2^{2^k} 
                        +
                     \|\matXtilde_{k-1}\|_2
                    \cdot 
                    \left(
                            1 + 2^{2^{k}+1} \cdot \|\matY_0\|_2^{2^{k+1}}
                    \right)
                    \right]
                    \nonumber\\
                &\leq 
                    \mu''(n)\cdot \umach
                    \cdot 
                    \left[
                        \psi(k,1,0)
                        \cdot (2\|\matY_0\|_2)^{2^k} 
                        +
                    (1+\|\matY_0\|_2)  \cdot 2^{2^k} \cdot 
                    \psi(k,2,1)
                    \cdot 
                    \left(
                            1 +  (2 \|\matY_0\|_2)^{2^{k+1}}
                    \right)
                    \right]
                    \nonumber
                \\
                &= 
                    \mu''(n)\cdot \umach
                    \cdot 
                    \phi(k),
        \end{align}
        where $\phi(k):=\left[
            \psi(k,1,0)
            \cdot (2\|\matY_0\|_2)^{2^k} 
            +
        (1+\|\matY_0\|_2)  \cdot 2^{2^k} \cdot 
        \psi(k,2,1)
        \cdot 
        \left(
                1 +  (2 \|\matY_0\|_2)^{2^{k+1}}
        \right)
        \right]$.
        Then the main error of Inequality \eqref{eq:bound_on_matf_k_initial} becomes:
        \begin{align}
            \|\matF_k\|_2 &\leq \mu''(n) \cdot \umach \cdot \left(
                \phi(k) + \sum_{j=1}^{k-1} \phi(j) \cdot \frac{\psi(k,2,0)}{\psi(j,2,0)}
            \right).
        \end{align}
    \end{proof}

\subsection{Proof of Newton-Schulz stability (Lemma \ref{lemma:stability_of_ns})}
\label{appendix:proof_of_stability_of_ns}
\begin{proof}
We analyze the algorithm that executes the floating point matrix operations per Newton-Schulz iteration based on the following order imposed by the parentheses:
\begin{align*}
    \matXhat_{k+1} \gets \matXhat_k(2\matI-(\matA\matXhat_k)),
\end{align*}
where $\matXhat_k$ is the matrix computed at iteration $k\geq 1$. We will denote by $\matX_k$ the corresponding matrix if the operations were executed in infinite precision. We have that $\matXhat_0=\matX_0$ (the initial matrix has no errors). For all subsequent matrices, we denote the accumulated error by:
\begin{align*}
    \matE_k:=\matXhat_k-\matX_k,
\end{align*}
and we also define the (exact) residual:
\begin{align*}
    \matR_k=\matI-\matA\matX_k.
\end{align*}
The following can be straightforwardly shown for all $k$:
\begin{align}
    &\|\matR_k\|\leq 1, & \|\matX_k\|\leq 2\|\matA^{-1}\|, & & \|\matA\|\|\matX_k\| \leq 2 \kappa(\matA), & & \|\matA\matX_k\| \leq 2.
    \label{eq:newton_schulz_proof_simple_bounds}
\end{align}
For the floating point operations, using \eqref{eq:elementwise_floating_point_bound} and \eqref{eq:matmul_floating_point_bound} we can write:
\begin{align*}
    \matY_k&:=\fl(\matA\matXhat_k) = \matA\matXhat_k+\matF_k, & \|\matF_k\|\
   &\leq \umach n^C \|\matA\|\|\matXhat_k\|,\\
    \matZ_k&:=\fl(2\matI-\matY_k)=2\matI-\matY_k+\matG_k, & \|\matG_k\| &\leq \umach\sqrt{n}\|2\matI-\matY_k\|,\\
    \matXhat_{k+1}&:=\fl(\matXhat_k\matZ_k)=\matXhat_{k}\matZ_k+\matH_k, & \|\matH_k\|&\leq \umach n^C \|\matXhat_k\|\|\matZ_k\|.
\end{align*}
By combining them together, we can write:
\begin{align}
    \matXhat_{k+1}&=\matXhat_k(2\matI-\matA\matXhat_k) - \matXhat_k\matF_k+\matXhat_k\matG_k+\matH_k \nonumber\\
    &=(\matX_k+\matE_k)(2\matI-\matA(\matX_k+\matE_k)) - (\matX_k+\matE_k)\matF_k+(\matX_k+\matE_k)\matG_k+\matH_k \nonumber\\
    &=\underbrace{\matX_{k}(2\matI-\matA\matX_k)}_{\matX_{k+1}}+2\matE_k-\matE_k\matA\matX_k-
     \matX_k\matA\matE_k - 
     \matE_k\matA\matE_k -(\matX_k+\matE_k)\matF_k+(\matX_k+\matE_k)\matG_k+\matH_k. \nonumber
\end{align}
By rearranging, we get an equation for the error at iteration $k$:
\begin{align}
    \matE_{k+1}=2\matE_k-\matE_k\matA\matX_k-\matX_k\matA\matE_k - \matE_k\matA\matE_k -(\matX_k+\matE_k)\matF_k+(\matX_k+\matE_k)\matG_k+\matH_k.
\end{align}

Combining with the bounds of \eqref{eq:newton_schulz_proof_simple_bounds} we can write:
\begin{align}
    \|\matE_{k+1}\| &\leq 2\|\matE_k\|
    +\|\matE_k\|\|\matA\matX_k\|
    +\|\matX_k\matA\| \|\matE_k\| 
    +\|\matA\|\|\matE_k\|^2
    + \|\matX_k+\matE_k\|  \|\matF_k+\matG_k\|
    +\|\matH_k\| 
    \nonumber\\
    &\leq 6\|\matE_k\| +\|\matA\|\|\matE_k\|^2 + \|\matX_k+\matE_k\|  \|\matF_k+\matG_k\|
    +\|\matH_k\|.\label{eq:newton_schulz_proof_simple_bounds_main_error_norm_inequality}
\end{align}

To move forward we expand appropriately the bounds for $\|\matF_k\|,$ $\|\matG_k\|$ and $\|\matH_k\|$. To simplify these bounds, we impose a first upper bound for $\umach$:
\begin{align}
    \umach\leq \frac{1}{n^{\max\{C, 1/2\}}\kappa(\matA)}.
    \label{eq:newton_schulz_proof_initial_umach_bound}
\end{align}
We have that:
\begingroup
\allowdisplaybreaks
\begin{align*}
    \|\matF_k\|\
       &\leq \umach n^C \|\matA\|\|\matXhat_k\|\\
       &\leq \umach n^C \|\matA\| \|\matX_k+\matE_k\|\\
       &\leq \umach n^C (2\kappa(\matA)+\|\matA\|\|\matE_k\|),\\
    \|\matG_k\| &\leq \umach\sqrt{n}\|2\matI-\matY_k\| \\
        &\leq \umach\sqrt{n}(2+\|\matA\matXhat_k+\matF_k\|) \\
        &\leq \umach\sqrt{n}\left(
            4+\|\matA\matE_k\|+
            \umach n^C (2\kappa(\matA)+\|\matA\|\|\matE_k\|)
        \right) \\
        &\leq \umach\sqrt{n}\left(
            6+2\|\matA\|\|\matE_k\|)
        \right) \\
    \|\matH_k\|&\leq \umach n^C \|\matXhat_k\|\|\matZ_k\|\\
        &\leq \umach n^C \|\matX_k+\matE_k\|\|2\matI-\matY_k+\matG_k\| \\
        &\leq \umach n^C \|\matX_k+\matE_k\|(\|2\matI-\matY_k\|+\|\matG_k\|) \\
        &\leq \umach n^C (\|\matX_k\|+\|\matE_k\|)
        \left((6+2\|\matA\|\|\matE_k\|)+\umach\sqrt{n}\left(
            6+2\|\matA\|\|\matE_k\|)
        \right)\right) \\
        &\leq \umach n^C (2\|\matA^{-1}\|+\|\matE_k\|)
        \left(
            12+4\|\matA\|\|\matE_k\|)
        \right) \\
\end{align*}

Plugging the previous bounds into \eqref{eq:newton_schulz_proof_simple_bounds_main_error_norm_inequality}, we obtain \begin{align*} 
\|\matE_{k+1}\| &\leq 6\|\matE_k\| +\|\matA\|\|\matE_k\|^2 \nonumber\\ &\quad+ \left(2\|\matA^{-1}\|+\|\matE_k\|\right) \left[ \umach n^C \left( 2\kappa(\matA) +\|\matA\|\|\matE_k\| \right) + \umach\sqrt n \left( 6+2\|\matA\|\|\matE_k\| \right) \right] \nonumber\\ &\quad+ \umach n^C\|\matA^{-1}\| \left( 2+\|\matA\|\|\matE_k\| \right) \left( 12+4\|\matA\|\|\matE_k\| \right). 
\end{align*} 
We now expand and collect separately the terms that are independent of $\|\matE_k\|$, linear in $\|\matE_k\|$, and quadratic in $\|\matE_k\|$. The terms independent of $\|\matE_k\|$ are: 
\begin{align*} \umach\|\matA^{-1}\| \left( 4n^C\kappa(\matA) +12\sqrt n +24n^C \right). 
\end{align*} 
The terms linear in $\|\matE_k\|$ are: 
\begin{align*} \left[ 6 + \umach \left( 24n^C\kappa(\matA) +4\sqrt n\,\kappa(\matA) +6\sqrt n \right) \right] \|\matE_k\|. \end{align*} 
Finally, the terms quadratic in $\|\matE_k\|$ are: 
\begin{align*} \|\matA\| \left[ 1 + \umach \left( n^C +2\sqrt n +4n^C\kappa(\matA) \right) \right] \|\matE_k\|^2. \end{align*} 
Therefore, 
\begin{align*} \|\matE_{k+1}\| &\leq \left[ 6 + \umach \left( 24n^C\kappa(\matA) +4\sqrt n\,\kappa(\matA) +6\sqrt n \right) \right] \|\matE_k\| \nonumber\\ &\quad+ \|\matA\| \left[ 1 + \umach \left( n^C +2\sqrt n +4n^C\kappa(\matA) \right) \right] \|\matE_k\|^2 \nonumber\\ &\quad+ \umach\|\matA^{-1}\| \left( 4n^C\kappa(\matA) +12\sqrt n +24n^C \right). 
\end{align*}
From \eqref{eq:newton_schulz_proof_initial_umach_bound}, the previous inequality can be simplified to
\begin{align} \|\matE_{k+1}\| &\leq \left( 6+34\umach n^{C'}\kappa(\matA) \right) \|\matE_k\| \nonumber\\ &\quad+ \|\matA\| \left( 1+7\umach n^{C'}\kappa(\matA) \right) \|\matE_k\|^2 \nonumber\\ &\quad+ 40\umach n^{C'}\kappa(\matA) \|\matA^{-1}\|, \label{eq:newton_schulz_proof_third_error_bound} \end{align}
where $C'=\max\{1/2,C\}$.

To move forward, notice that if $
    \|\matA\|\|\matE_k\| \leq 1,$
then we have that:
\begin{align}
    \|\matA\|\|\matE_k\|^2
    \leq
    \|\matE_k\|,
\end{align}
and therefore:
\begin{align}
    \|\matE_{k+1}\|
    &\leq
    40\|\matE_k\|
    +
    8\|\matA\|\|\matE_k\|^2
    +
    40\umach n^{C'}\kappa(\matA)\|\matA^{-1}\|
    \nonumber\\
    &\leq
    48\|\matE_k\|
    +
    40\umach n^{C'}\kappa(\matA)\|\matA^{-1}\|.
    \label{eq:newton_schulz_error_simplified}
\end{align}

Let us now combine everything together. Suppose that we want to execute the algorithm for $t$ iterations and bound the final error $\|\matE_t\|$. We impose the following additional bound on $\umach$ (which is stricter than \eqref{eq:newton_schulz_proof_initial_umach_bound}):
\begin{align}
\umach
\leq
\frac{1}{
40n^{C'}\kappa^2(\matA)
\displaystyle\sum_{j=0}^{t-1}48^j
}.
\label{eq:newton_schulz_final_umach_bound}
\end{align}
Since $\kappa(\matA)\geq 1$ and $\sum_{j=0}^{t-1}48^j\geq 1$, this bound also implies
\begin{align}
\umach n^{C'}\kappa(\matA)\leq 1.
\end{align}

We now verify by induction that, for every $k=0,\ldots,t$,
\begin{align}
\|\matA\|\|\matE_k\|\leq 1.
\label{eq:newton_schulz_final_induction}
\end{align}
The claim is immediate for $k=0$, since $\matE_0=0$. Assume now that
$\|\matA\|\|\matE_j\|\leq1$ 
for every $j=0,\ldots,k-1$. We can therefore apply
\eqref{eq:newton_schulz_error_simplified} at every previous iteration, obtaining
\begin{align}
\|\matE_k\|
&\leq
40\umach n^{C'}\kappa(\matA)\|\matA^{-1}\|
\sum_{j=0}^{k-1}48^j.
\label{eq:newton_schulz_error_unrolled}
\end{align}
Multiplying by $\|\matA\|$ gives
\begin{align}
\|\matA\|\|\matE_k\|
&\leq
40\umach n^{C'}\kappa^2(\matA)
\sum_{j=0}^{k-1}48^j
\nonumber\\
&\leq
40\umach n^{C'}\kappa^2(\matA)
\sum_{j=0}^{t-1}48^j
\nonumber\\
&\leq 1,
\end{align}
where the last inequality follows from
\eqref{eq:newton_schulz_final_umach_bound}.
Hence \eqref{eq:newton_schulz_final_induction} holds for all
$k=0,\ldots,t$, and we can safely use
\eqref{eq:newton_schulz_error_simplified} throughout all $t$ iterations. In particular,
\begin{align}
\|\matE_t\|
&\leq
40\umach n^{C'}\kappa(\matA)\|\matA^{-1}\|
\sum_{j=0}^{t-1}48^j.
\label{eq:newton_schulz_final_error}
\end{align}
For $t=O(\log n)$,
\begin{align*}
\sum_{j=0}^{t-1}48^j
\leq
48^t
=
n^{O(1)}.
\end{align*}
Therefore, there exists a constant $C''$ such that
\begin{align*}
\|\matE_t\|
\leq
\umach n^{C''}\kappa(\matA)\|\matA^{-1}\|.
\end{align*}

\endgroup
\end{proof}

\end{document}